\documentclass[sigconf,nonacm]{acmart}

\usepackage{amsmath,amsfonts}

\usepackage{amsthm}

\newtheorem*{prob_state}{\textbf{Problem Statement}}

\newtheorem{theorem}{\textbf{Theorem}}

\usepackage{float}

\usepackage{color}
\usepackage{tikz}
\usetikzlibrary{trees}
\usepackage{amsthm}
\usepackage{makecell}
\usepackage{tikz}
\usetikzlibrary{trees}
\usepackage{tabularx,colortbl}
\usepackage{threeparttable}
\usepackage{booktabs}
\usepackage{multirow}
\usepackage{subfig}
\usepackage{graphicx}
\usepackage{url}
\usepackage{enumitem}
\usepackage{tcolorbox}
\usepackage{xcolor}
\usepackage[ruled, vlined, linesnumbered]{algorithm2e}
\usepackage{hyperref}
\usepackage[noabbrev]{cleveref}
\crefname{equation}{Eq.}{Eqs.}
\usepackage{adjustbox}

\usepackage[normalem]{ulem}

\newcommand{\filledcircle}{\tikz\fill[black] (0,0) circle (.8ex);}
\newcommand{\emptycircle}{\tikz\draw (0,0) circle (.8ex);}

\definecolor{DeepPink}{HTML}{FF1493}
\definecolor{Orchid}{HTML}{DA70D6}
\definecolor{Magenta}{HTML}{FF00FF}
\definecolor{Fuchsia}{HTML}{FF00FF}
\definecolor{LavenderPink}{HTML}{FFB6C1}
\definecolor{verylightgray}{rgb}{0.9, 0.9, 0.9}
\definecolor{lightred}{rgb}{1,0.8,0.8}

\AtBeginDocument{%
	}

\begin{document}
	
	\title{EVE: Efficient Verification of Data Erasure through Customized Perturbation in Approximate Unlearning}
	
	\author{Weiqi Wang}
	%\email{larst@affiliation.org}
	\affiliation{%
		\institution{University of Technology Sydney}
		\city{Sydney}
		\state{NSW}
		\country{Australia}}
	
	\author{Zhiyi Tian}
	%\email{larst@affiliation.org}
	\affiliation{%
		\institution{University of Technology Sydney}
		\city{Sydney}
		\state{NSW}
		\country{Australia}}

	\author{Chenhan Zhang}
	%\email{larst@affiliation.org}
	\affiliation{%
		\institution{University of Technology Sydney}
		\city{Sydney}
		\state{NSW}
		\country{Australia}}
	
		\author{Luoyu Chen}
	%\email{larst@affiliation.org}
	\affiliation{%
		\institution{University of Technology Sydney}
		\city{Sydney}
		\state{NSW}
		\country{Australia}}
	
	\author{Shui Yu}
	%\email{larst@affiliation.org}
	\affiliation{%
		\institution{University of Technology Sydney}
		\city{Sydney}
		\state{NSW}
		\country{Australia}}

	\begin{abstract} 
Verifying whether the machine unlearning process has been properly executed is critical but remains underexplored. Some existing approaches propose unlearning verification methods based on backdooring techniques. However, these methods typically require participation in the model's initial training phase to backdoor the model for later verification, which is inefficient and impractical. In this paper, we propose an efficient verification of erasure method (EVE) for verifying machine unlearning without requiring involvement in the model's initial training process. The core idea is to perturb the unlearning data to ensure the model prediction of the specified samples will change before and after unlearning with perturbed data. The unlearning users can leverage the observation of the changes as a verification signal. Specifically, the perturbations are designed with two key objectives: ensuring the unlearning effect and altering the unlearned model's prediction of target samples. We formalize the perturbation generation as an adversarial optimization problem, solving it by aligning the unlearning gradient with the gradient of boundary change for target samples. We conducted extensive experiments, and the results show that EVE can verify machine unlearning without involving the model's initial training process, unlike backdoor-based methods. Moreover, EVE significantly outperforms state-of-the-art unlearning verification methods, offering significant speedup in efficiency while enhancing verification accuracy. The source code of EVE is released at  \uline{\url{https://anonymous.4open.science/r/EVE-C143}}, providing a novel tool for verification of machine unlearning.

\end{abstract}    
	



	\maketitle
	\thispagestyle{plain}
	\pagestyle{plain}

\section{Introduction}
\label{sec:intro}

With increasing concerns about personal data privacy, regulations like the GDPR~\cite{mantelero2013eu} have been introduced. These legal frameworks establish the ``right to be forgotten,'' enabling individuals to request the removal of their data from machine learning (ML) systems. This mandate has spurred research into \textit{machine unlearning}, which seeks to develop techniques for eliminating the influence of specific user data from trained ML models~\cite{bourtoule2021machine,cao2015towards,neel2021descent,warnecke2024machine}. While substantial progress has been made in devising unlearning algorithms \cite{cao2015towards,thudi2022necessity,sekhari2021remember}, much of the focus remains on algorithmic development, with relatively little attention given to the critical aspect of unlearning verification to ensure that data has indeed been effectively removed. 

%Especially for the unlearning users, they have unlearning requirements, but they are challenging to implement verification as they usually only have a black-box access to the models \citep{gao2024verifi,sommer2020towards}.

%\vspace{2mm}
%\noindent
%\textbf{Research Gap.} 

In existing studies, the backdoor-based methods are common solutions for the unlearning users to conduct verification of the data erasure of unlearning~\citep{hu2022membership,guo2023verifying,sommer2022athena}. However, these backdoor-based methods are inefficient in practice, as they require embedding backdoors during the initial model training phase. The effectiveness of these methods depends on incorporating backdoors from the outset~\cite{hu2022membership,guo2023verifying}, but it is impractical for users to anticipate the need to unlearn specific data samples during the model's original training \cite{wang2023mm,lin2020composite}. Moreover, such methods introduce unnecessary inefficiencies, as they involve the training process rather than focusing exclusively on verifying the unlearning operation \citep{wang2025tape,thudi2022necessity}. Therefore, we explore the following question: \textit{``Is there a more efficient and practical solution that verifies unlearning independently of the original model training process to ensure both efficiency and effectiveness in real-world applications?''}

%A more practical and efficient solution would be to design a verification method that verifies unlearning independently of the original training process, ensuring both efficiency and effectiveness in real-world applications.

%\noindent
%\textbf{Research Question.} 
%Given the current state of research, we pose the research question: \textit{``When an unlearning request is uploaded and processed, can we provide a practical verification mechanism to confirm the effective removal of specified data?''} To ensure practicality, the verification method should focus exclusively on the unlearning process, independent of the original model training.% For effectiveness, it must rigorously confirm whether the unlearning operation is conducted and the specified data has been unlearned.

\noindent
\textbf{Motivation.} 
Machine unlearning aims to remove the learned information of user-requested samples from the trained model and result in an unlearned model~\cite{chen2021machine,kurmanji2024towards,hu2024eraser,tarun2023fast}. We utilize the property that users upload the unlearning data as requests to design methods to perturb the unlearning samples to further remove some target information, ensuring the model decision boundary will shift for the target samples after unlearning. We utilize prediction changes for targeted samples as a signal for unlearning verification. In this way, the verification of unlearning focuses solely on the unlearning process, significantly enhancing its efficiency and practicality. We illustrate the comparison between our method and existing approaches, along with the underlying intuition, in~\Cref{fig_unlearningauditfigure1}.

\begin{figure}[t] 
	\centering 
	\includegraphics[width=0.99\linewidth]{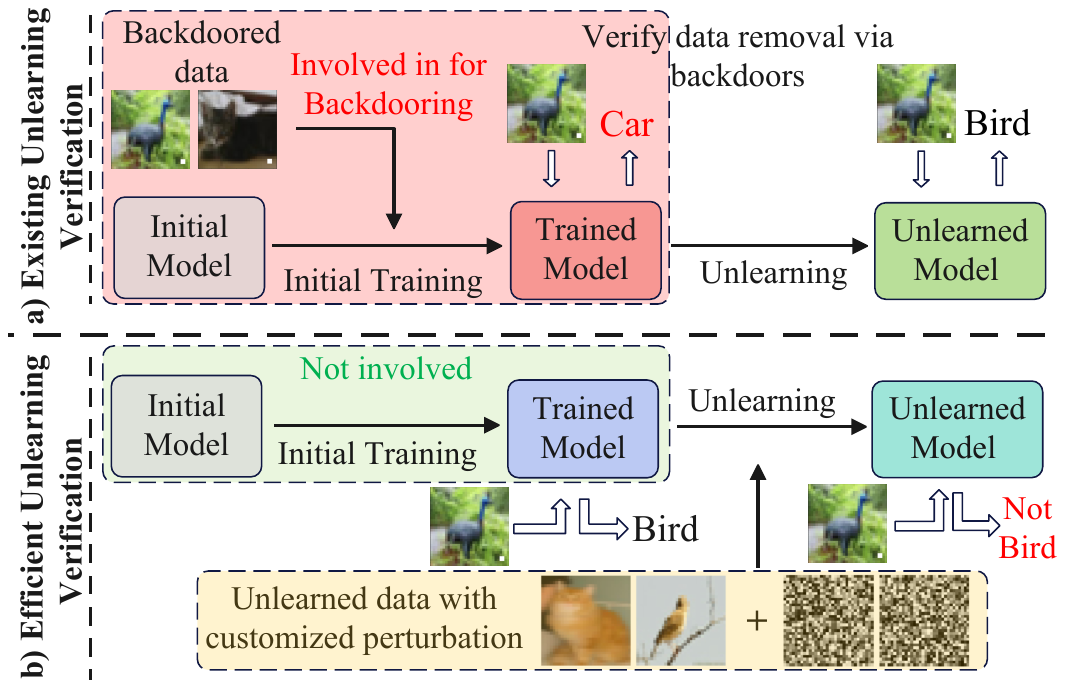}
	\vspace{-2mm}
	\caption{(a) The backdoor-based verification and (b) the intuition of efficient verification of erasure (EVE) based on the perturbed unlearning data. The scheme only involves the unlearning process rather than the original model training process.}
	\label{fig_unlearningauditfigure1}
	%	\vspace{-2mm}
\end{figure}

\noindent
\textbf{Our Work.} 
In this paper, we propose an efficient verification of data erasure method (EVE), which achieves verification without involving the initial model training process, ensuring efficiency and practicality. The general idea of EVE is shown at the bottom in \Cref{fig_unlearningauditfigure1}, where we aim to find a method that shifts the model decision boundary for specified samples after executing machine unlearning. Hence, the unlearning users can utilize the prediction changes of the specified samples as the verification signal for unlearning. In EVE, we plan to apply customized perturbations to the unlearning data to induce a shift in the unlearned model's decision boundary for specified samples. One challenge is that the customized perturbation may mitigate the unlearning effect. To solve this, we define two objectives of the perturbation: ensuring the effectiveness of the unlearning process and altering the model's prediction of the verification signal. Then, we formalize the perturbation generation as an adversarial optimization problem, solving it by aligning gradients of the unlearning operation with gradients of the boundary shift for verification-specified samples. Moreover, to provide theoretical guarantees for the verification results, we introduce an innovative approach by applying statistical hypothesis testing~\cite{montgomery2010applied} to determine whether the model has been successfully unlearned.

To evaluate the proposed method, we conduct extensive experiments across three mainstream approximate unlearning benchmarks and four representative datasets. The results demonstrate the clear superiority of EVE over state-of-the-art unlearning verification methods~\cite{hu2022membership,guo2023verifying,wang2025tape} in both efficiency and effectiveness. Notably, EVE achieves a significant speedup across all datasets compared to backdoor-based methods~\cite{hu2022membership,guo2023verifying}, as EVE is independent of the initial training process, much more efficient than backdoor-based methods that embed backdoors during the initial model training. Meanwhile, EVE achieves a comparable verification effect to the backdoor-based verification methods.

%With significant efficiency gains, EVE delivers verification performance comparable to or better than backdoor-based approaches. 

We summarize our contribution as follows:
\begin{itemize}[itemsep=0pt, parsep=0pt, leftmargin=*]
	\item We investigate verification for data erasure exclusively involving the unlearning process, setting it apart from existing backdoor-based approaches, which depend on embedding backdoors during the model's initial training phase. This distinction makes our method more efficient and practical. 
	\item We propose the EVE method, which perturbs the unlearning data to embed a signal into unlearned models for verification. The perturbation generation is formalized as an adversarial optimization problem, and we propose a perturbation descent method based on gradient matching to solve it. This approach ensures the accurate detection of unlearning and the effectiveness of unlearning.
	\item We innovatively employ hypothesis testing to offer statistical guarantees for the verification results and present the corresponding theoretical analysis. 
	 % requiring only black-box access to the target unlearned model.
	\item We conduct extensive experiments, and the results confirm the substantial efficiency improvements over state-of-the-art unlearning verification methods and the effectiveness of our approach in diverse settings. 
	% , shedding light on the design of future unlearning verification methods. 
\end{itemize}

%\noindent \textbf{Roadmap.} We introduce the related work in Section \ref{related_work}. We introduce the threat model and formalize the machine unlearning verification problem in Section \ref{problem_df}. Section \ref{muv_method} presents the details of our proposed COP for unlearning verification. In \Cref{exp_setting,exp_eval,exp_abla}, we introduce the experimental setting and illustrate the experimental results to evaluate the proposed COP. We summarize the paper in Section \ref{s_a_fw}. 

	%\cite{xu2024really}

\begin{table}[t]
	% \tiny
	\scriptsize
	\centering
%	\vspace{-2mm}
	\caption{ \small Overview of unlearning verification methods. %\vspace{-2mm}
	}
	\vspace{-2mm}
	\label{overview_of_auditing_method}
	\resizebox{\linewidth}{!}{
		\setlength\tabcolsep{2.pt}
		\begin{tabular}{c|cccc}
			\toprule[1pt]
			\multirow{2}{*} { \makecell[c]{\textbf{Unlearning} \\ \textbf{Verification} \\ \textbf{Methods}} } & \multicolumn{2}{c} {\textbf{Involving Processes}  } & \multicolumn{2}{c} { \textbf{Target Model}}   \\
			\cmidrule(r){2-3}   \cmidrule(r){4-5}  
			& \makecell[c]{{Original training} \\ {and unlearning	}  }  & \makecell[c]{{Only unlearning } \\ {process }  } & \makecell[c]{{Original trained} \\ {models	}  }    & \makecell[c]{{Unlearned} \\ {models}  }   \\ 
			\midrule
			MIB~\cite{hu2022membership} &\filledcircle & \emptycircle & \filledcircle&\emptycircle	  \\
			Athena~\cite{sommer2022athena} &\filledcircle & \emptycircle & \filledcircle  &\emptycircle    \\
			Veri. in the dark~\cite{guo2023verifying} &\filledcircle & \emptycircle & \filledcircle  &\emptycircle    \\
			Verifi~\cite{gao2024verifi} &\filledcircle &\emptycircle & \filledcircle  &\emptycircle     \\
			IndirectVerify \citep{xu2024really} &\filledcircle &\emptycircle &  \emptycircle &  \filledcircle   \\
			TAPE \citep{wang2025tape} &\emptycircle &\emptycircle &  \filledcircle &  \filledcircle   \\
			EVE (Ours)	     &\emptycircle & \filledcircle & \emptycircle  & \filledcircle     \\
			\bottomrule[1pt]
	\end{tabular}}
	%\vspace{-2mm}
	\begin{tabbing}
		\filledcircle: denotes the verification method is applicable; \\ \emptycircle: denotes the verification method is not applicable.
	\end{tabbing}
	\vspace{-4mm}
\end{table}

\section{Related Work} \label{related_work}

Though there are many studies focusing on investigating unlearning algorithms \cite{chen2021machine,nguyen2020variational,wang2023machine,kurmanji2024towards}, only a few studies addressed the essential challenge of providing verification for unlearning users' data~\cite{thudi2022necessity,wang2025tape,xu2024really}. Backdoor-based methods were the main solutions for unlearning users to verify data removal in machine unlearning~\cite{hu2022membership,sommer2022athena,guo2023verifying,gao2024verifi}. These approaches involve embedding backdoored samples into user data during model training, allowing verification of unlearning by checking whether the backdoor is removed from the models~\cite{hu2022membership,guo2023verifying}. Recently \citep{xu2024really} constructs the influential sample pairs: trigger samples and reaction samples for model training, so that users can use reaction samples to verify if the unlearning operation has been successfully carried out. However, these methods are limited by their dependence on the initial model training process, which is impractical in real-world applications where users cannot predict which data may need to be unlearned in the future. Besides, recent study \citep{wang2025tape} proposed unlearning verification based on posterior difference before and after unlearning, which cuts off the reliance on the original model training but still need the assistance of the trained model to calculate the posterior difference. 

Our EVE approach is significantly different from existing unlearning verification methods~\cite{hu2022membership,sommer2022athena,guo2023verifying,gao2024verifi} in terms of the involved processes and models, which is summarized and depicted in \Cref{overview_of_auditing_method}.

\section{Threat Model and Problem Statement} \label{problem_df}

%In this section, we introduce the ML scenario and the threat model, then the unlearning verification problem statement and the requirements for unlearning verification.

% while ensuring effective data removal and compliance.

% To comply with the ``right to be forgotten'' legislation and foster a privacy-preserving environment, the ML server is responsible for conducting machine unlearning operations. However, verifying unlearning remains challenging, as it is difficult to ensure that the server has genuinely processed the unlearning and not merely claimed compliance. 
%\noindent

\begin{figure*}[h]
	\centering
	\includegraphics[width=0.98\linewidth]{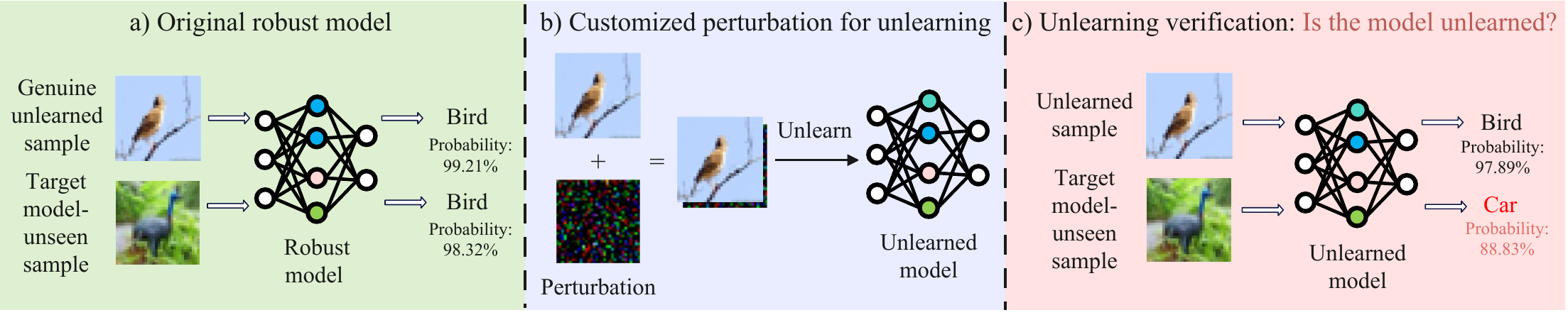}
	\vspace{-2mm}
	\caption{An example of the idea of our method. a) A trained robust model can correctly predict the genuine unlearned and target model-unseen samples with high probability. b) Our method customizes perturbation for the erased data for unlearning, aiming to make the unlearned model easily misclassify the unseen data after unlearning. c) The unlearning user queries for specified samples to verify if the unlearned model is influenced by the perturbation.
	}
	\label{fig:figure2}
%	\vspace{-2mm}
\end{figure*}

\subsection{MLaaS Scenario and Threat Model}

\noindent
\textbf{MLaaS Scenario.} We present our unlearning verification problem within the context of a Machine Learning as a Service (MLaaS) scenario. For clarity and ease of understanding, only two entities are involved: the ML server, which collects user data, trains models, and provides machine learning services, and the users, who upload data to servers for model training. Users may invoke their right to be forgotten, requiring the ML server to unlearn their data from the trained models.

\noindent
\textbf{Threat Model.} For the ML server, following typical unlearning verification settings~\citep{hu2022membership,guo2023verifying,wang2025tape}, we suppose the server is honest during model initial training but may use output suppression~\citep{yao2024large,lu2022quark,jang_etal_2023_knowledge} instead of a truly machine-unlearning operation to preserve model performance~\citep {cooper2024machine}. This scenario is more plausible than forging unlearning updates~\cite{thudi2022necessity}, as simulating data removal and managing the resulting utility loss would demand significant effort with little benefit to the server.

For the unlearning users, they have a local dataset containing the samples to be erased but lacks access to the full training data~\cite{warnecke2024machine,hu2024eraser}. The users can only query the model using their data. Similar to other common settings~\cite{hu2023duty,huang2025unlearn,di2022hidden}, we assume the unlearning algorithms are jointly determined by the server and users. Hence, the unlearning users can craft perturbations to unlearning samples for verification while not influencing unlearning effectiveness.

%Additionally, we assume the unlearning algorithms are jointly determined by the server and users, which is commonly set in other studies~\cite{hu2023duty}. Using these, the unlearning users can craft perturbations to the unlearning samples, potentially inserting a verification signal into the unlearned model for verification.

% To implement our method, we assume the unlearning users are capable of reconstructing the same ML model as the target ML service model in terms of architecture. This can be achieved through model hyperparameter stealing methods~\cite{wang2018stealing,Seong2018towards,salem2020updates}, which allow the user to replicate the model.

\subsection{Data Erasure Verification Problem in Machine Unlearning} \label{pro_state}

%To clarify the unlearning verification problem, we first outline the key stages of the machine unlearning process.

We first revisit the key stages of the machine unlearning process.

\noindent
\textbf{Machine Unlearning.} 
The process of unlearning typically involves three steps: (1) The ML server initially trains a model with parameters $\theta^*$ derived from the dataset $D$. (2) The unlearning user submits a request to the server, providing the dataset $D_u$ that they wish to have unlearned. (3) The ML server then applies an unlearning method, $\mathcal{U}$, to remove the influence of $D_u$ from the model, resulting in an unlearned model with parameters $\theta_{u, D \backslash D_u}$, where $D \backslash D_u$ represents the remaining dataset after the removal of $D_u$.

%, ensuring that the specified data has indeed been removed
 
The goal of machine unlearning verification is to assess whether the unlearning operation has been executed correctly and honestly. We formalize this problem as follows. 
\begin{prob_state}[Machine Unlearning Verification] 
	\label{verification_problem}	
	Given the potential for the ML server to spoof the execution of unlearning, and the capabilities of the unlearning user, unlearning verification requires the user to devise a method to infer whether the unlearning operation has been honestly performed. Specifically, the user must verify that the specified data $D_u$ has been properly unlearned, resulting in a transition from the trained model parameters $\theta^*$ to the unlearned parameters $\theta_u$.
\end{prob_state} 
It is essential to clarify that, in the statement, our verification is solely on the transition from the trained model $\theta^*$ to the unlearned model $\theta_u$. We try to solve this problem by perturbating the unlearning data to introduce a verification signal in the unlearned model for later verification. An effective customized perturbation for unlearning verification should meet the following requirements. 
\begin{itemize}[itemsep=0pt, parsep=0pt, leftmargin=*]
	\item \textit{Unlearning verification effect.} To ensure the verification effect for unlearning, we must ensure the signal of boundary shift is effective on the unlearned model when the server conducts unlearning with the perturbed samples. 
	%Meanwhile, the perturbed samples should not achieve the boundary shift effect on other methods, such as continual learning.
	%	\item \textit{Unlearning effect of the perturbed samples.} 
	%\item \textit{Stealthiness of the perturbed samples.} The perturbed unlearning samples for unlearning verification is not easy to be distinguished from normal unlearning samples.
	%\item \textit{Unlearning effect and functionality preservation.} The perturbed unlearning samples should still be effective for unlearning and will not come at the cost of significant model utility degradation.
	\item \textit{Unlearning effect and functionality preservation.} The perturbed unlearning samples should still be effective for unlearning and will not come at the cost of significant model utility degradation.
	\item \textit{Stealthiness of the perturbed unlearning sample.} The perturbed unlearning sample for verification is not easy to be distinguished from normal unlearning samples.
\end{itemize}

%Importantly, the target sample is chosen from outside the training dataset, which helps mitigate potential model utility degradation that may occur during the unlearning and signal embedding process.

%, verifying the proper execution of the unlearning process.

% we do not focus on the transition from the initial model $\theta$ to the trained model $\theta^*$, a phase required by existing backdoor-based verification methods. Instead,

\section{EVE: Efficient Verification of Data Erasure} \label{muv_method}

\subsection{Overview of the EVE Method}

This approach leverages the fact that users must submit unlearning requests to the server, allowing users to prepare and upload the data to be unlearned. Based on this, we propose unlearning verification through customized perturbation on unlearning data, which aims to cause the unlearned model to misclassify the specified target samples that were previously correctly classified by the trained robust model. We treat the prediction changes of the target samples before and after unlearning as the verification signal. We present an example of the idea and the solution process in \Cref{fig:figure2}. The general methodology includes three main steps: a) choosing model-unseen target samples which has high prediction probability; b) customizing perturbation for unlearning data and executing unlearning, which ensures both unlearning effectiveness and prediction shift of the target model-unseen samples; c) the unlearning users query for target samples to verify the unlearning status.

%to introduce a verification signal specifically targeting the unlearning process. The perturbation aims

\subsection{Customized Perturbation for Unlearning Data}

Recall the requirements of the unlearning verification defined in \Cref{pro_state}, we should guarantee the unlearning verification effect (prediction shift of target samples), the unlearning effect, and the stealthiness of the perturbation (constraint perturbation distance). Considering the unlearning algorithm, we can formalize the unlearning verification as an optimization problem, as outlined below

%\vspace{-4mm}
\begin{equation} \label{perturbation_aim}
%	\small
	\begin{aligned}
 	\min_{\delta }  \sum_{t=1}^T  \mathcal{L}_{u}( (x_t, y \neq y_t); \theta_{u}),  \\
    \text{ s.t. } \theta_{u} \in \arg \min_{\theta}  \sum_{i \leq M} \mathcal{L}_{u} ((x_i + \delta_i, y_i);\theta ),
	\end{aligned}
\end{equation}
where $M$ is the size of $D_u$, $ \mathcal{L}_{u}$ is the loss function of the agreed-upon unlearning algorithm, and $T$ is the size of the target samples (used for verification). \Cref{perturbation_aim} aims to identify a perturbation for the erased samples, $(x_i + \delta_i, y_i)$, such that, after unlearning these perturbed samples, the unlearned model alters its prediction of target samples $(x_t, y_t)$ that were correctly predicted by the original trained model. 
We define the constraint $\delta: \| \delta \|_{\infty} \leq d$ to ensure that the perturbed data remains similar to the original unlearning data to ensure the stealthiness.

%We define the constraint that $\delta: \| \delta \|_{\infty} \leq \alpha$ to ensure that the perturbed data will not be too different from the original unlearning data.

\noindent
\textbf{Unlearning Data Perturbation by Gradient Matching.}
Our objective is to find a perturbation $\delta$ such that, when the model is unlearned with the customizedly perturbed samples $D_{u}^p$, it minimizes the loss of errorly predicting the target samples, meanwhile minimizing the unlearning loss. Directly solving \Cref{perturbation_aim} is computationally intractable due to the bilevel nature of the optimization objective. Instead, one may implicitly make the update of unlearning approximate to the update of misclassifying target samples by finding a suitable $\delta$ such that for any model parameter $\theta$, the following condition is satisfied: 
\begin{equation} \label{noise_approx}
	\small
%	\scriptsize, _{\hat{y}_t \neq  y_t}
%\tiny
	\begin{aligned}
		\sum_{t=1}^{T} \nabla_{\theta} \mathcal{L}_{u} ( (x_t, y \neq y_t);\theta_{u})  \approx  \sum_{ i =1}^{M} \nabla_{\theta} \mathcal{L}_{u} ((x_i + \delta_i,y_i);\theta).
	\end{aligned}
\end{equation} 
If we can enforce \Cref{noise_approx} to hold for any $\theta$ during unlearning, the gradient steps that minimize the loss on the erased samples will also minimize the mislabeling loss for the target samples. Consequently, unlearning the model based on data with perturbation $\delta$ will ensure that the model mispredicts the target signal data $(x_t, y_t)$. However, calculating $\delta$ that satisfies \Cref{noise_approx} for all values of $\theta$ is intractable. A potential solution, as suggested in \cite{geiping2020witches,di2022hidden}, is to relax \Cref{noise_approx} so that it holds for a fixed, pre-trained model. We align the gradients of the unlearning and verification signal embedding targets by minimizing their negative cosine similarity. Specifically, we fix the trained parameters $\theta$ and use the perturbation $\delta$ to optimize the similarity between two gradients as 
\begin{equation} \label{cosine_sim}
	\small
	%\scriptsize
	\begin{aligned}
	 \phi (\delta, \theta) = 1 - \frac{\langle  \sum \nabla_{\theta} \mathcal{L}_{u} ( (x_t, y \neq y_t);\theta_{u}),  \sum \nabla_{\theta} \mathcal{L}_{u} ((x_i + \delta_i,y_i);\theta) \rangle}{ \|  \sum  \nabla_{\theta} \mathcal{L}_{u} ( (x_t, y \neq y_t);\theta_{u}) \| \cdot  \| \sum \nabla_{\theta} \mathcal{L}_{u} ((x_i + \delta_i,y_i);\theta) \|} .
	\end{aligned}
\end{equation}
\citeauthor{geiping2020witches}~\cite{geiping2020witches} solve it by initially randomizing $\delta$ and performing $K$ steps of Adam optimization to minimize $\phi(\delta, \theta_{D_{u}^{p}})$. They employ $R$ restarts, typically $R \leq 10$, to improve the robustness of the perturbation. Using this approach, we can also identify suitable perturbations for verification signal embedding via unlearning. However, this method is not always effective, as it does not consistently yield satisfactory perturbations within 10 restarts. To overcome this limitation, we propose an unlearning data perturbation descent strategy.

\noindent
\textbf{Perturbation Descent.} 
We initialize a perturbation matrix, $\delta$, and treat it as parameters to be optimized, while keeping the model parameters, $\theta$, fixed. We add the perturbation to the unlearning data $D_u$, using the perturbed data as input to the model. Next, we calculate the minimization loss using \Cref{cosine_sim}, which measures the cosine similarity between the unlearning and the verification signal embedding gradients. Using this loss, we apply gradient descent to update the perturbation matrix $\delta$, while keeping the model parameters $\theta$ fixed. After several iterations of training, this approach generates sufficient unlearning data perturbation, enabling the unlearned model to achieve the desired verification signal embedding effect. The detailed algorithm for the unlearning data perturbation descent is presented in \Cref{Noise_backpropagation_update}.

\begin{algorithm}[t]
	%\small
	\caption{Unlearning Data Perturbation Descent (UDPD)} \label{Noise_backpropagation_update}
	\begin{small} % small, normalsize
		\BlankLine
		\KwIn{Trained model $\theta$, verification signal target $(x_t, y_t)$, unlearning dataset $D_u$ }
		\KwOut{The synthesized data with the unlearning perturbation, $D_{u}^{p} = (X_a + \delta, Y_a)$} %, and three metrics for verification
		\SetNlSty{}{}{} % This line removes the vertical line before the for-loop
		\SetKwFunction{UDPD}{\textbf{UDPD}}
		\SetKwProg{Fn}{procedure}{:}{end procedure}
		\SetNlSty{}{}{} % This line removes the vertical line before the for-loop
		\Fn{\UDPD{$\theta$, $(x_t,y_t)$, $D_u$}}{
			$\delta_{1} \gets \mathcal{N}(0,1)$  \hspace{4mm}    $\rhd$ Initialize unlearning perturbation. \\
			$\nabla \theta^t \gets  \nabla_{\theta} \mathcal{L}_{unl.}((x_{t}, y \neq y_t); \theta)$  \hspace{2mm} $\rhd$ Compute gradients for target. \\
			\For{$i \gets 1$ \KwTo $n$}{
				$X_{u,i}^p \gets X_u + \delta_i$  \hspace{2mm} $\rhd$ Add the perturbation to data. \\
				$\nabla \theta_i \gets  \nabla_{\theta} \mathcal{L}_{unl.}((X_{u,i}^p, Y_u); \theta_i)$  \hspace{2mm} $\rhd$ Compute gradients for unlearning. \\
				$\phi_i  \gets \text{Sim}( \nabla \theta^t, \nabla \theta_i)$  \hspace{2mm} $\rhd$ Compute similarity using \Cref{cosine_sim}. \\
				$\delta_{i+1} \gets \delta_{i} - \eta \nabla_{\delta} ( \phi_i)$   \hspace{2mm} $\rhd$ Update perturbation to match gradients. \\
			}
			\Return $D_{u}^{p} = (X_u + \delta_{n+1}, Y_u)$
		}
	\end{small}
\end{algorithm}

% the illustrated process is presented in \Cref{fig_noiseadjustingthroughbackpropagation}.

%Here, we have a question, should we keep the model fixed, or also simulating the model training process.

%When we solving the optimization problem, we fix the original model parameters that was trained based on the clean samples and add

The Unlearning Data Perturbation Descent (UDPD) algorithm iteratively perturbs unlearning data by adding perturbations to the original unlearning dataset to align the unlearned model’s behavior with the verification objective. It begins by initializing a random perturbation and calculating the gradient of the unlearning loss for the verification signal target. 
Note that, even in black-box setting where the user cannot access the model, he/she can calculate the gradient with model stealing methods \citep{tramer2016stealing,orekondy2019knockoff,ilyas2018black}.
In each iteration, the algorithm updates the perturbation, computes gradients for the unlearning objective, and measures the similarity between the gradients of the verification objective and the unlearning objective. The perturbation is then refined based on this similarity until the process converges.

%\subsection{Executing Unlearning}

\noindent
\textbf{Executing Unlearning.} 
In machine unlearning, users typically submit an unlearning request that includes specified samples to be removed. In our approach, users upload the perturbed unlearning data to the server for unlearning. The server, in agreement with the users, applies the mutually selected unlearning algorithms to unlearn the provided data. We denote the associated loss function as $\mathcal{L}_{u}$. Once the unlearning process is complete, the server releases a black-box API of the updated model to users to provide ML services. The unlearning users can query the model with their verification samples and infer the conduction of unlearning from the changed prediction of the target samples.

\subsection{Unlearning Verification}

Since the unlearning users have perturbed their data, if the server has honestly executed the specified unlearning algorithms, the unlearned model should reflect the embedded verification signal. If this signal is absent, it indicates that the unlearning process has not been properly carried out. Therefore, users can verify the unlearning of their data by demonstrating the model's expected behavior, confirming that the unlearning operation was correctly executed.

To provide a statistical guarantee for the unlearning verification, we employ statistical testing~\citep{montgomery2010applied} to estimate the confidence level in determining whether the unlearned model has been successfully embedded with the verification signal. Specifically, we apply hypothesis testing to assess whether the unlearning process was correctly executed. We define the unlearning confirmation hypothesis $\mathcal{H}_1$ and the null hypothesis $\mathcal{H}_{0}$ as follows: 
\begin{equation}
	\begin{aligned}
		\mathcal{H}_1 \colon  \Pr ( f (x_t) \neq y_t) \geq  \beta, \\
		\mathcal{H}_0 \colon  \Pr (f (x_t) \neq y_t) \leq \beta,
	\end{aligned}
\end{equation}
where $\Pr (f (x_t) \neq y_t)$ denotes the probability that the model has not correctly predicted the target sample $f(x_t)$ as $y_t$, and $\beta$ represents the trained model's misprediction probability before unlearning. In this paper, we set $\beta = \frac{K-1}{K}$, which reflects random chance, where $K$ is the number of classes in the classification task. Our experiments demonstrate that the misprediction probability of the trained model before unlearning is significantly lower than $\frac{K-1}{K}$ as the trained model is robust to unseen samples.

The unlearning confirmation hypothesis $\mathcal{H}_1$ states that the verification signal for unlearning with the misprediction probability is significantly higher than random chance, indicating a clear difference between the behavior of the unlearned model and the original trained model. In contrast, the null hypothesis $\mathcal{H}_{0}$ posits that the misprediction probability is less than or equal to random chance, meaning the unlearned model behaves similarly to the trained model before unlearning. If users can confirm hypothesis $\mathcal{H}_{1}$ with statistical confidence, they can verify that their data is unlearned. Otherwise, they can conclude that their unlearning request was not properly executed.

Because the unlearning users are only given black-box access to the target model, they can query the model and achieve the final results of the verification signal probability $\alpha = \Pr (f(x_t) \neq y_t)$. Based on the verification probability, we consider that the data owner can use a T-test \cite{montgomery2010applied} to test the hypothesis for adversarial samples. Below, we formally stated under what conditions the unlearning user can reject the null hypothesis $\mathcal{H}_{0}$ at the significance level $1 - \tau$ (i.e., with $\tau$ confidence) using a limited number of queries.

\begin{theorem} \label{theorem_1}
	Given a target model $f(\cdot)$ in a classification task with $K$ classes, and $m$ queries to $f(\cdot)$, if the model's misprediction probability $\alpha$ for the target sample $f(x_t) \neq y_{t}$ satisfies the following formula:
	\begin{equation}
		\sqrt{m-1}  \cdot (\alpha - \beta) - \sqrt{\alpha -\alpha^2} \cdot t_{\tau}> 0,
	\end{equation}
	the unlearning user can reject the null hypothesis $\mathcal{H}_{0}$ at significance level $1-  \tau$, 
	where $\beta = \frac{K-1}{K}$ is the expected misprediction probability under random chance, and $t_{\tau}$ is the $\tau$ quantile of the $t$ distribution with $m-1$ degrees of freedom.
\end{theorem}

%\Cref{proof_of_theorem1}

The proof of the \Cref{theorem_1} can be found in \Cref{proof_of_theorem1}. \Cref{theorem_1} implies that if the test statistic exceeds the critical value $t_{\tau}$, the unlearning user can reject the null hypothesis at the given significane level $1- \tau$, indicating that the model is likely to execute the unlearning operation and thus influenced by the perturbation for the specific sample $x_t$.

\section{Performance Evaluation} \label{exp_setting}

%Experimental Settings

%In this section, we will introduce the settings, the general evaluation results, and the ablation studies.

\subsection{Experimental Settings}

%\subsection{Datasets and Models}
\noindent
\textbf{Datasets.}
We conducted experiments on four representative public image datasets: MNIST \cite{deng2012mnist}, CIFAR10 \cite{krizhevsky2009learning}, STL-10 \cite{coates2011analysis}, and CelebA \cite{liu2018large}, covering a range of levels of learning complexity. We present the statistics of all datasets and the corresponding detailed introduction in \Cref{datasets_appendix}.

%We conducted experiments on four widely adopted public datasets: MNIST \cite{deng2012mnist}, CIFAR10 \cite{krizhevsky2009learning}, STL-10 \cite{coates2011analysis}, and CelebA \cite{liu2018large}. MNIST, CIFAR10, and STL-10 are benchmark datasets utilized for 10-class image classification tasks, offering a range of objective categories with varying levels of learning complexity. The statistics of all datasets and detialed introduction are listed and introduced in \Cref{datasets_appendix}. %used in our experiments

\noindent
\textbf{Models.}
Three representative model architectures are selected in our experiments: a 5-layer MLP with ReLU activations, a ResNet-18, and a 7-layer convolutional neural network (CNN). Particularly, we train a 5-layer MLP model on MNIST, a ResNet18 on CIFAR10 and STL-10, and a 7-layer CNN on CelebA. All methods are implemented using Pytorch and are conducted on NVIDIA Quadro RTX 6000 GPUs. 

%On the MNIST dataset, we set the learning rate $\eta = 0.001$. On CIFAR10, STL-10, and CelebA, we set the learning rate $\eta=0.0001$. During training, we set the minibatch size to 16 on MNIST and CIFAR10, the minibatch size to 2 on STL-10, and the minibatch size to 160 on CelebA. 

 \begin{table*}[t]
	% \tiny
	%\scriptsize
	\caption{Overall Evaluation Results. The best results are highlighted in \textbf{Bold}. % \vspace{-2mm} 
	}
	\label{tab_total}
	\resizebox{\linewidth}{!}{
		\setlength\tabcolsep{2.pt}
		\begin{tabular}{c|cccc|cccc|cccc|cccc}
			\toprule[1pt]
			\multirow{2}{*} {\makecell[c]{\textbf{Metrics} }} & \multicolumn{4}{c} {MNIST, $\text{\it ESR}=2\%$}& \multicolumn{4}{c} {CIFAR10, $\text{\it ESR}=2\%$} & \multicolumn{4}{c} {STL-10, $\text{\it ESR}=2\%$} & \multicolumn{4}{c} {CelebA, $\text{\it ESR}=0.6\%$} \\
			\cmidrule(r){2-5}   \cmidrule(r){6-9} \cmidrule(r){10-13} \cmidrule(r){14-17}
			& Original & MIB    & TAPE	 & EVE & Original & MIB   & TAPE & EVE	& Original	 &   MIB   & TAPE & EVE & Original	 &   MIB   & TAPE & EVE  \\
			\midrule 
			RT (s)   & 142  & 146   & 26.5  &  \textbf{12.81}    & 599   & 613 		 & 99.3	& \textbf{3.93} 	 & 506	& 560   & 54.3 & \textbf{3.44}   &  592	& 641   & 23.6 & \textbf{7.52}  \\
			UV (\%)  	  & -  & $100.0\%$  & $94.97\%$ & \textbf{100.0\%}     & - & $100.0\%$  & 96.20\% & \textbf{100.0\%}  	& - & $100.0\%$  & 84.40\% & \textbf{100.0\%} &  -  &  $92.89\%$   & 88.42\% &  \textbf{100.0\%} \\
			AO (\%)      & \textbf{98.91\%}   & 98.73\%    &  \textbf{98.91\%}  & \textbf{98.91\%}    & \textbf{79.98\%}    & 77.67\%  & \textbf{79.98\%}   & \textbf{79.98\%} & \textbf{67.27\%} &  66.61\%  &  \textbf{67.27\%} &  \textbf{67.27\%}  &  \textbf{96.24\%} &  95.98\%   & \textbf{96.24\%} &  \textbf{96.24\%} \\
			AU (\%)      & -   &97.51\%     &  \textbf{98.16\%} &   {97.69\%}    & -   & 77.36\%    & \textbf{79.45\%} & {79.35\%} & -  &  66.57\%  &  \textbf{67.19\%} &  {67.17\%}  & - &  95.79\%   & \textbf{96.02\%} &  {95.84\%} \\
			\bottomrule[1pt]
	\end{tabular}}
	%\vspace{-2mm}
	\scriptsize
	\begin{tabbing}
		RT: the running time to assess the efficiency; VU: the unlearning verifiability; AO: the accuracy of the original trained model; \\AU: the accuracy of the unlearned model.
	\end{tabbing}
	%\vspace{-6mm}
\end{table*}

%\subsection{Metric}

\noindent
\textbf{Metric.}
We use the verifiability metric to measure the verification effect, the model accuracy to evaluate the model utility preservation, and the running time to assess efficiency. To summarize, we have the following three metrics: 
\begin{itemize}[itemsep=0pt, parsep=0pt, leftmargin=*]
	\item \textbf{Unlearning Verifiability (UV).} Verifiability measures the data removal verification by calculating the misclassification rate of the unlearned model for the verification target samples, denoted as \[\textbf{UV} = \frac{1}{T} \sum_t^T  \mathbb{I}(f(x_t) \neq y_t),\] where $\mathbb{I}$ is the indicator function that equals 1 when its argument is true ($f(x_y) \neq y_t$) and 0 otherwise. 
	\item \textbf{Model Accuracy.} Model accuracy is used to evaluate utility preservation and shows whether the verification methods degrade the utility of the ML model, which includes the accuracy of the original trained model (\textbf{AO}) and the accuracy of the unlearned model (\textbf{AU}). 
	\item  \textbf{Running Time (RT).}  It assesses the efficiency, which records the running time of the entire process of each method.
\end{itemize}

%\subsection{Compared Unlearning Verification and Machine Unlearning Benchmarks}

\noindent
\textbf{Unlearning Verification Benchmarks.}
There are a number of unlearning verification methods~\cite{wang2025tape,hu2022membership,sommer2022athena,guo2023verifying}, mostly relying on backdooring techniques and involving the original model training processes. We choose one representative backdoor-based method, MIB~\citep{hu2022membership}, and one verification method without backdoor, TAPE~\citep{wang2025tape}, as the comparison baselines.

%Among them, MIB~\cite{hu2022membership} is the most representative and achieves the best verification performance. Therefore, we only choose the MIB method to compare with our method.

\noindent
\textbf{Unlearning Benchmarks.} 
We choose three mainstream approximate unlearning algorithms to evaluate the unlearning verification methods, specifically, HBU~\cite{guo2019certified}, VBU~\cite{nguyen2020variational}, and SalUn~\cite{fan2024salun}. We briefly summarize these unlearning benchmarks as follows. %in \Cref{benchmarks_intro}.
\begin{itemize}[itemsep=0pt, parsep=0pt, leftmargin=*]
	%\item \textbf{RFU \cite{wang2023machine}.} RFU is an approximate unlearning method. It tries to unlearn a bottleneck representation by minimizing the mutual information between the representation and the erased samples. We set a middle layer of the original model as the representation layer and implement unlearning according to \cite{wang2023machine}.
	\item  \textit{Hessian matrix-based Unlearning \textbf{(HBU)}} \cite{sekhari2021remember}: HBU is an approximate unlearning method, which needs to calculate the inverse Hessian matrix of the remaining dataset as the weight for an unlearning update. We implement HBU  follow the unlearning process as introduced in \cite{sekhari2021remember}. 
	\item \textit{Variational Bayesian Unlearning \textbf{(VBU)}} \cite{nguyen2020variational}: VBU is an approximate unlearning method based on variational Bayesian inference. For the convenience of experiments, we set a middle layer of original neural networks as the Bayesian layer and calculate the unlearning loss according to \cite{nguyen2020variational} based on the Bayesian layer and erased samples for unlearning. 
	\item \textit{Saliency Unlearning \textbf{(SalUn)}} \citep{fan2024salun}: SalUn introduces ``weight saliency'' to remove the influence of samples and classes for unlearning a model, improving effectiveness and efficiency.
\end{itemize}

%\Cref{benchmarks_intro} of Appendix.

%\noindent

\subsection{Overview Evaluation of EVE}
We first present an overall evaluation of different verification methods in \Cref{tab_total}, which is evaluated for the VBU \cite{nguyen2020variational} unlearning method. Bolded values highlight the best performance for each metric. A dash indicates that a specific evaluation metric is not applicable to the corresponding method.

\noindent
\textbf{Setup.}
We set the Erased Sample Rate (\textit{ESR}) to $2\%$ for MNIST, CIFAR-10, and STL-10, and to $\text{\it ESR}=0.6\%$ for CelebA due to its significantly larger dataset size compared to the other three. To better illustrate the balance between utility preservation and efficiency, we report the performance of the original model without any unlearning and verification methods, shown in the ``Original'' column in \Cref{tab_total}. We also record the accuracy of both the original trained and unlearned models to assess the influence of different methods on the model.

\noindent
\textbf{Evaluation of Efficiency.}
As EVE does not involve the initial model training process, it significantly improves efficiency compared to MIB and ``Original.'' The ``Original'' column represents the original model training without any verification methods. MIB involves backdooring the model during its initial training phase, adding considerable overhead. Compared with MIB, EVE achieves the highest $160\times$ speedup on STL-10. Though TAPE also does not rely on the initial training period, it needs to mimic the shadow unlearning models to prepare data and train a reconstructor for unlearning verification, which consumes more computation time than EVE.

\noindent
\textbf{Unlearning Verification Effect and Functionality Preservation.} 
In terms of verifiability, EVE consistently achieves 100\% unlearning verifiability across all datasets, matching or exceeding MIB and TAPE's performance, particularly on CelebA, where EVE outperforms MIB and TAPE (100\% vs. 92.89\% and 88.42\%). EVE and TAPE maintain the same accuracy as ``Original'' because it does not involve the original training process. Meanwhile, EVE maintains high accuracy for the unlearned models. For example, on MNIST, the original model accuracy is 98.91\%, and after unlearning, EVE preserves an accuracy of 97.69\%, which is much better than MIB. Since TAPE also does not influence the unlearned model, it has a slightly higher model accuracy than EVE, but the cost is that the verifiability of TAPE is not as stable as EVE.
These results illustrate that EVE not only provides reliable unlearning verification but also ensures minimal degradation of model utility, as indicated by the small accuracy drop.

\iffalse
\begin{tcolorbox}[colback=white, boxrule=0.3mm]
	\noindent \textbf{Takeaway 1.} CAP achieves the best efficiency and model functionality preservation as our scheme is independent of the initial model training. Moreover, CAP provides similar or even better unlearning verifications as backdoor-based methods. 
\end{tcolorbox}
\fi 

%(both how much information is unlearned and data removal status verification)

%\subsection{Evaluations on Various Unlearning Methods} \label{exp_on_unl_benchmarks}

%\vspace{2mm}
%\noindent
%\textbf{Setup.} 
%We set the value of {\it ESR} the same as in \Cref{tab_total}. When evaluating the EVE, to keep the setting similar to MIB, we add the perturbation with the same perturbation distance as the trigger patch of MIB.  
%We only perturb the unlearned data through \Cref{cosine_sim} but do not change the labels as MIB. We demonstrate the evaluation of unlearning verification methods on three mainstream approximate unlearning benchmarks in \Cref{fig_test06all}.

%To straightforwardly display the compared verification results of having or not having unlearned specified samples, we correspondingly set two situations, i.e., the erased samples $D_{u,b}$ and $D_{u,nois}$ not in or still in the remaining dataset. 

\begin{figure}[t]
	\centering
	\includegraphics[width=0.99\linewidth]{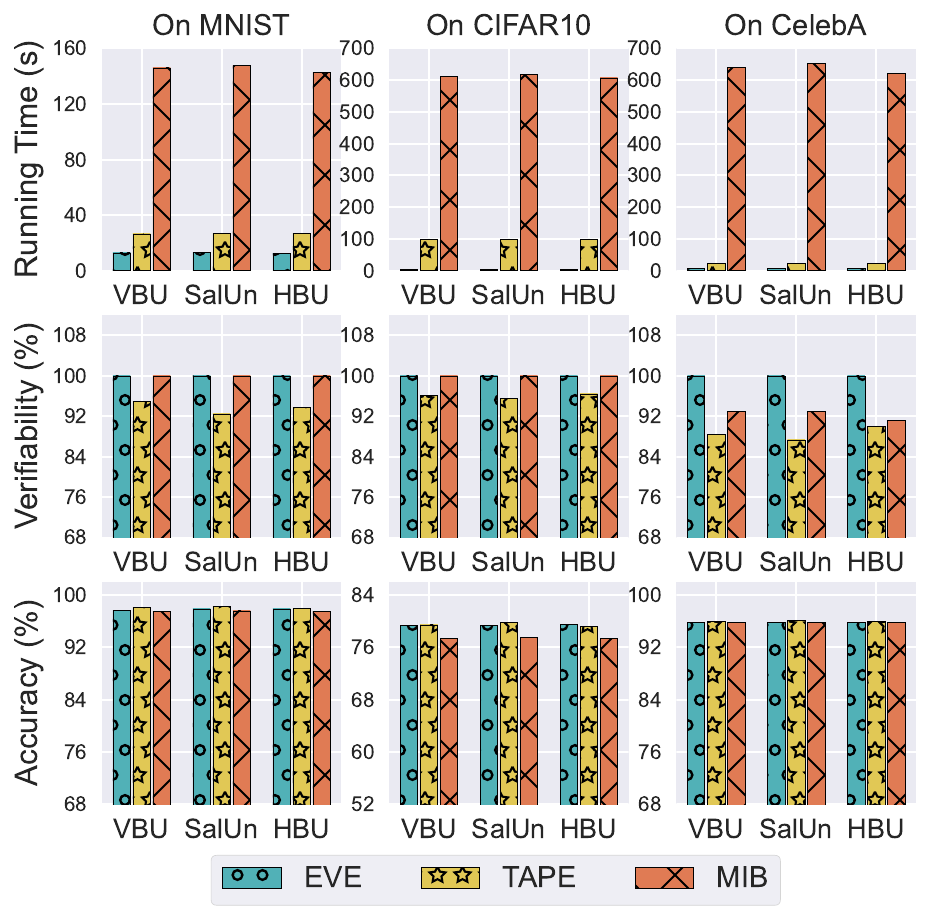}
	\vspace{-2mm}
	\caption{Data erasure verification for different unlearning methods. \vspace{-4mm}
		%EVE consistently achieves significant efficiency improvement, a better unlearning verification effect, and functionality preservation.
	} 
	\label{fig_test06all}
\end{figure}

\subsection{Evaluations of Verification on Various Unlearning Methods}
We also present the evaluation of verification methods for three mainstream approximate unlearning benchmarks in \Cref{fig_test06all}, where the experimental settings are the same as the above. The top row demonstrates the running time for verification of different unlearning methods across different datasets. 
EVE has the best efficiency and significantly outperforms MIB, consistently achieving faster execution times than MIB and TAPE across all datasets and unlearning methods (VBU, SalUn, HBU). On MNIST, EVE reduces the running time to approximately 12.8 seconds compared to MIB's 146 seconds. On CIFAR-10 and CelebA, EVE achieves 3.93 seconds and 7.52 seconds, respectively, showcasing substantial speedups over MIB. This confirms that EVE provides a more time-efficient solution for unlearning verification, largely due to its avoidance of the initial model training process.

%\noindent
%\textbf{Evaluations of Efficiency.} 

%\noindent
%\textbf{Verification Effect and Functionality Preservation.} 
The second row of \Cref{fig_test06all} presents the verifiability of unlearning for EVE, TAPE, and MIB. EVE maintains perfect verifiability (100\%) across all experiments, consistently exceeding TAPE and slightly better than MIB. The third row of \Cref{fig_test06all} shows the accuracy preservation, which only contains the unlearned model accuracy here. Since TAPE is not involved in both the learning and unlearning process, it achieves the best model accuracy. EVE consistently preserves model accuracy better than MIB across all datasets and methods. 

%On MNIST, EVE retains high accuracy, matching or slightly improving upon MIB. On CIFAR-10 and CelebA, EVE maintains minimal utility loss after unlearning. These results highlight EVE's ability to preserve the model functionality while verifying unlearning.

\iffalse
\begin{table}[h]
	% \tiny
	%\scriptsize
	\caption{  Overall Evaluation Results on MNIST, CIFAR10, STL-10, and CelebA. \vspace{-2mm} 
	}
	\label{with_additional_user}
	\resizebox{\linewidth}{!}{
		\setlength\tabcolsep{3.pt}
		\begin{tabular}{c|cccccc}
			\toprule[1pt]
			\multirow{2}{*} {\makecell[c]{\textbf{Metrics} }} & \multicolumn{3}{c} {MNIST}& \multicolumn{3}{c} {CIFAR10}  \\
			\cmidrule(r){2-4}   \cmidrule(r){5-7}
			& Original & MIB    	 & EVE & Original & MIB   & EVE\\
			\midrule 
			Running time (s)   & 142  & 146    &  \textbf{12.81}    & 599   & 613 			& \textbf{3.93}  \\
			Verifiability   	  & -  & $100.0\%$ & \textbf{100.0\%}     & - & $100.0\%$ & \textbf{100.0\%}  \\
			Accuracy     & \textbf{98.91\%}   & 98.73\%    & \textbf{98.91\%}    & \textbf{79.98\%}    & 77.67\%    & \textbf{79.98\%} \\
			\bottomrule[1pt]
	\end{tabular}}
	\vspace{-2mm}
%	\scriptsize
%	\begin{tabbing}
%		Accuracy (Original): the accuracy of the original trained model; Accuracy (Unlearned): the accuracy of the unlearned model.
%	\end{tabbing}
%	\vspace{-4mm}
\end{table}

\fi

\begin{figure}[t]
	\centering
	\includegraphics[scale=0.4]{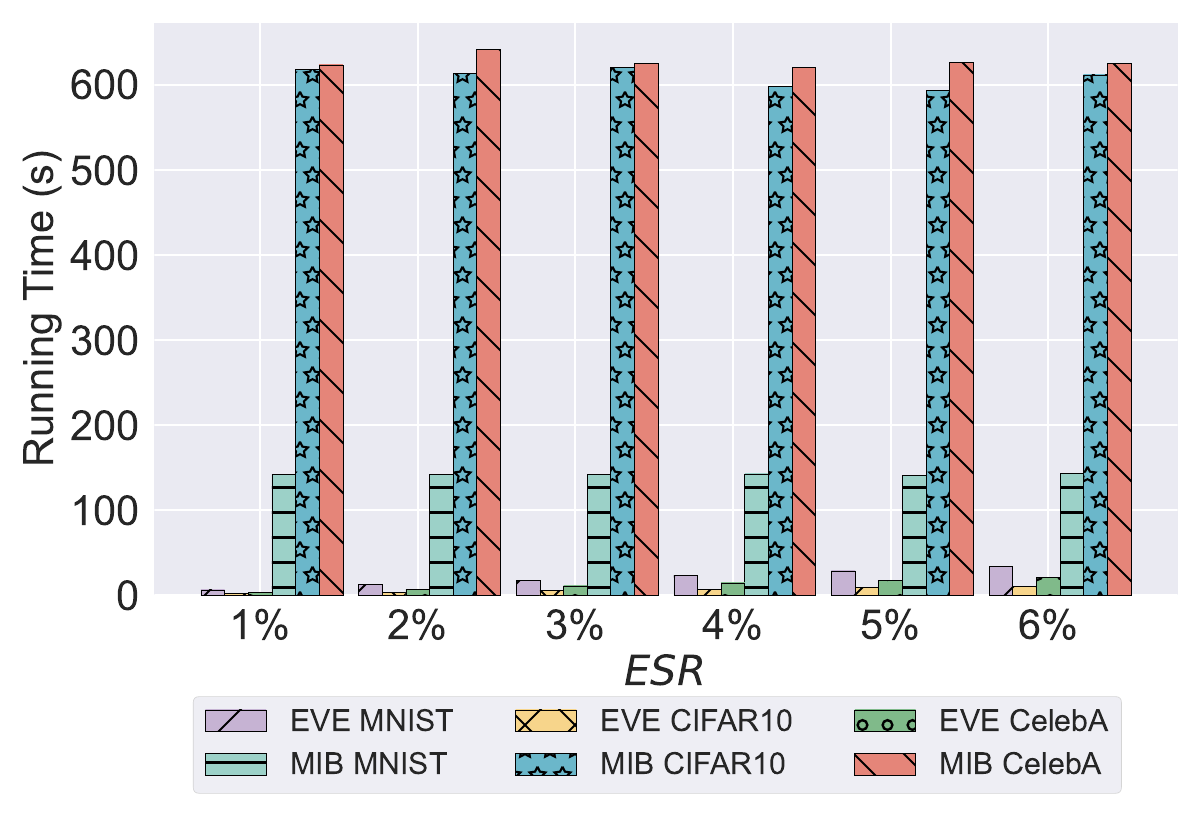}
	\vspace{-2mm}
	\caption{Efficiency impacted by different $\text{\it ESR}$.}
	\label{evaluation_of_running_time} 
	\vspace{-4mm}
\end{figure}

\subsection{Scalability of Verification Methods on Various $\text{\it ESR}$ } \label{impact_of_ess}

\noindent
\textbf{Setup.} 
We set the $\text{\it ESR}$ on MNIST and CIFAR10 from $1\%$ to $5\%$ to evaluate the impact of $\text{\it ESR}$. We keep other parameters fixed when evaluating $\text{\it ESR}$. The ablation study is still conducted based on VBU \cite{nguyen2020variational}. Additional scalability evaluation about unlearning additional users' data is presented in \Cref{add_experiment_of_scala}. 

%Moreover, since our EVE always has the lowest computation cost than TAPE and MIB, as shown in \Cref{evaluation_of_running_time}.

\noindent
\textbf{Evaluation of Efficiency.} 
The main components of the running time of EVE are perturbing the unlearning samples and unlearning executing, which are highly related to the size of the erased samples. \Cref{evaluation_of_running_time} illustrates the impact of varying the {\it ESR} on the running time for both EVE and MIB across MNIST, CIFAR-10, and CelebA. As the {\it ESR} increases, the running time for both methods gradually rises. However, EVE consistently demonstrates significantly lower running times compared to MIB, regardless of the {\it ESR} level. For example, at an {\it ESR} of 2\%, EVE requires less than 10 seconds for CIFAR10 and CelebA, while MIB takes over 500 seconds on both datasets. The figure also shows that while the running time for EVE increases slightly when {\it ESR} increases, it remains consistently faster than MIB across all datasets. This trend highlights EVE's efficiency in handling larger unlearning requests, as it avoids the extensive overhead required by MIB, i.e., needing to backdoor during the initial model training phase.

%\Cref{evaluation_of_in_or_not_in} illustrates the impact of $\text{\it ESR}$ when providing unlearning verification on MNIST, CIFAR10 and CelebA. The $\text{\it ESR}$ on MNIST and CIFAR10 is from $1\%$ to $5\%$ and on CelebA is from $0.5\%$ to $0.9\%$. The ablation study is still conducted based on the representative unlearning method, VBU \cite{nguyen2020variational}. 

\begin{figure}[t]
	\centering
	\subfloat{ 	  \rotatebox{90}{ \hspace{8mm}  { On MNIST} }
		\includegraphics[scale=0.28]{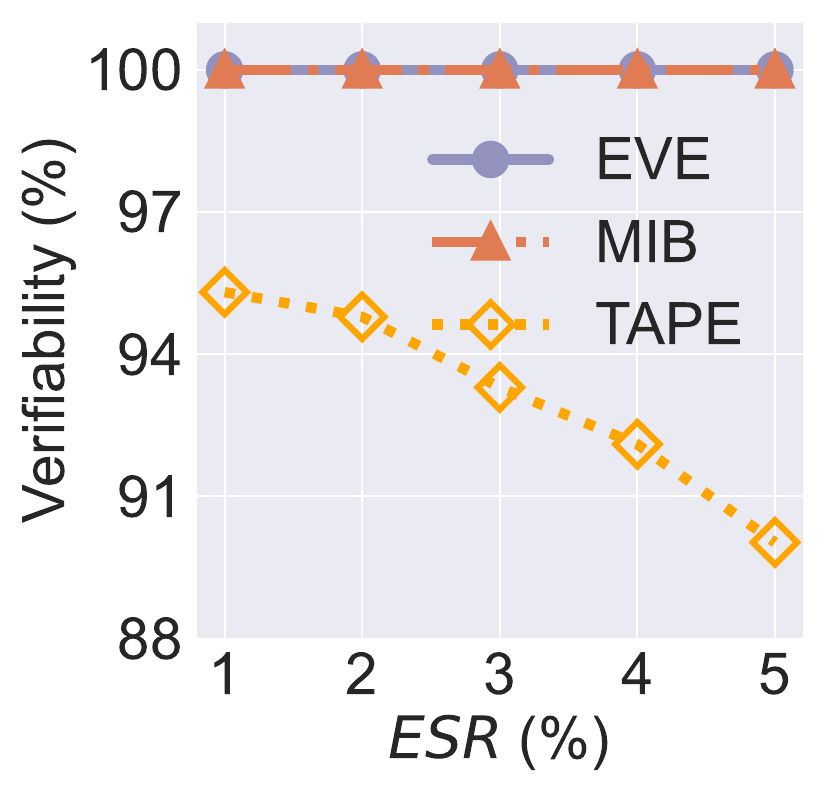}
		%\caption{An example of a subfigure.}
		\label{fig:mnistaccbetacurve}
	}
	\subfloat{ 	 
		\includegraphics[scale=0.28]{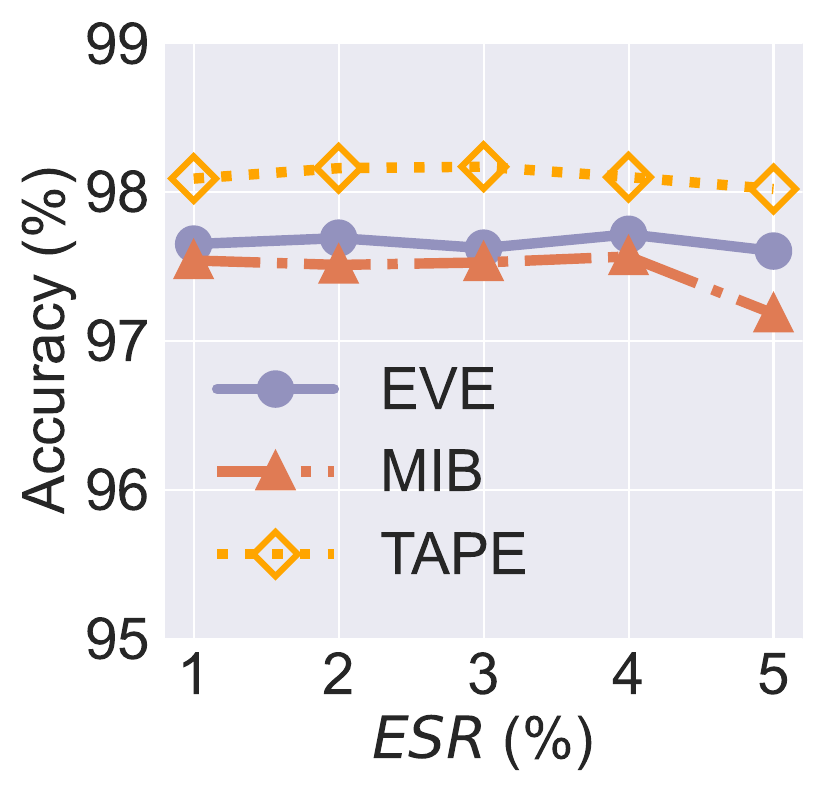}
		%\caption{An example of a subfigure.}
		\label{fig:mnistbackaccbetacurve} 
	}
	\\
	\vspace{-4mm}
	\subfloat{ 	 \rotatebox{90}{  \hspace{8mm} { On CIFAR10} }
		\includegraphics[scale=0.28]{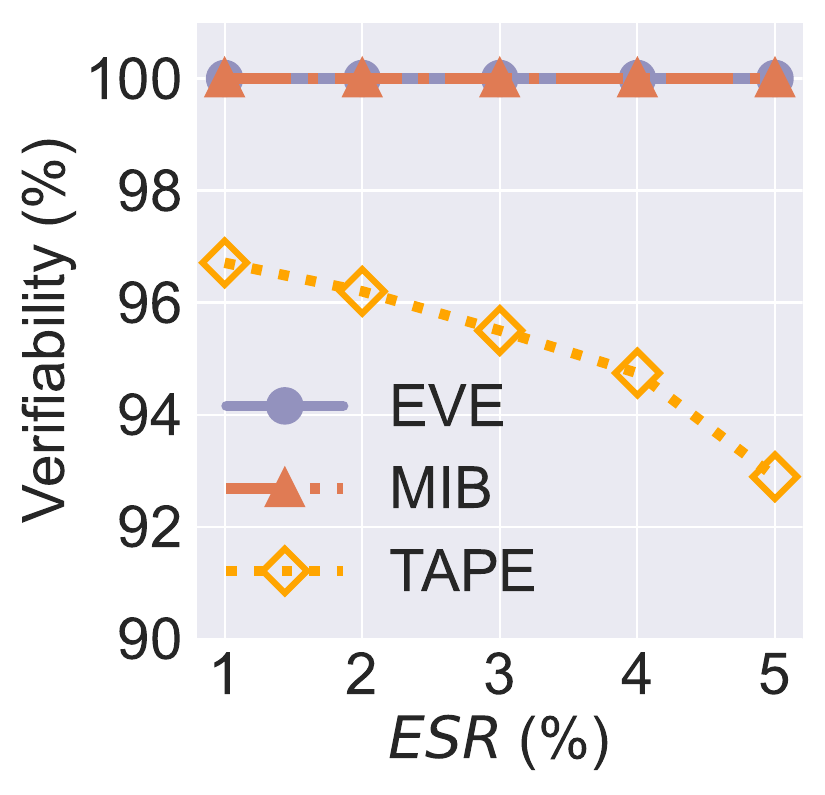}
		%\caption{An example of a subfigure.}
		\label{fig:cifar10verifiabilitysamplesizesimanalysis}
	}
	\subfloat{ 	 %
		\includegraphics[scale=0.28]{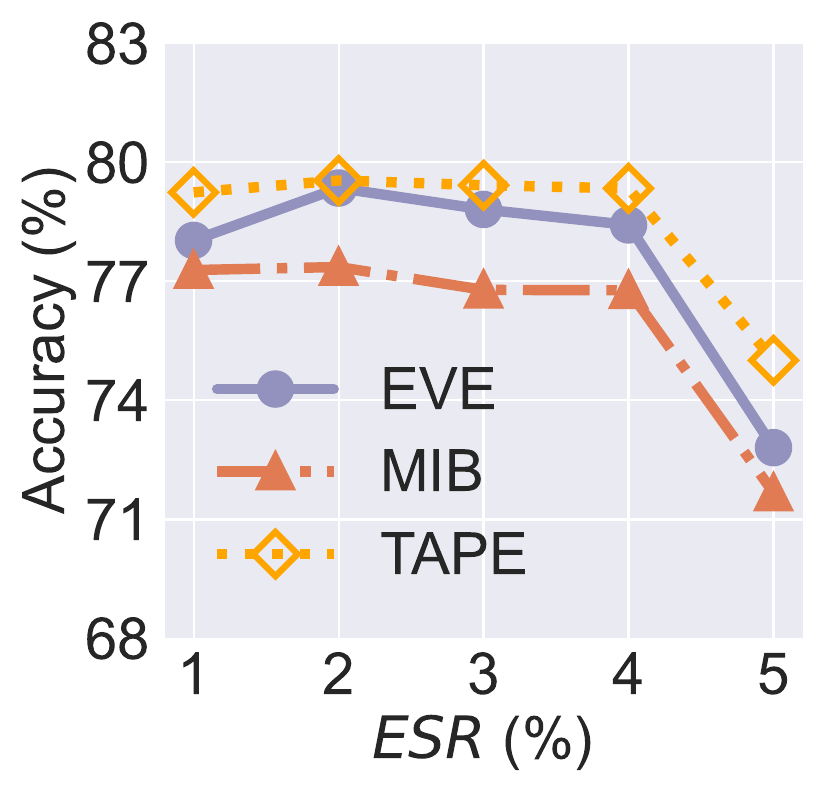}
		%\caption{An example of a subfigure.}
		\label{fig:cifar10recsimsamplesizesimanalysis}
	} \\ 	\vspace{-4mm}
	\subfloat{ 	\rotatebox{90}{  \hspace{10mm} { On CelebA} }
		\includegraphics[scale=0.28]{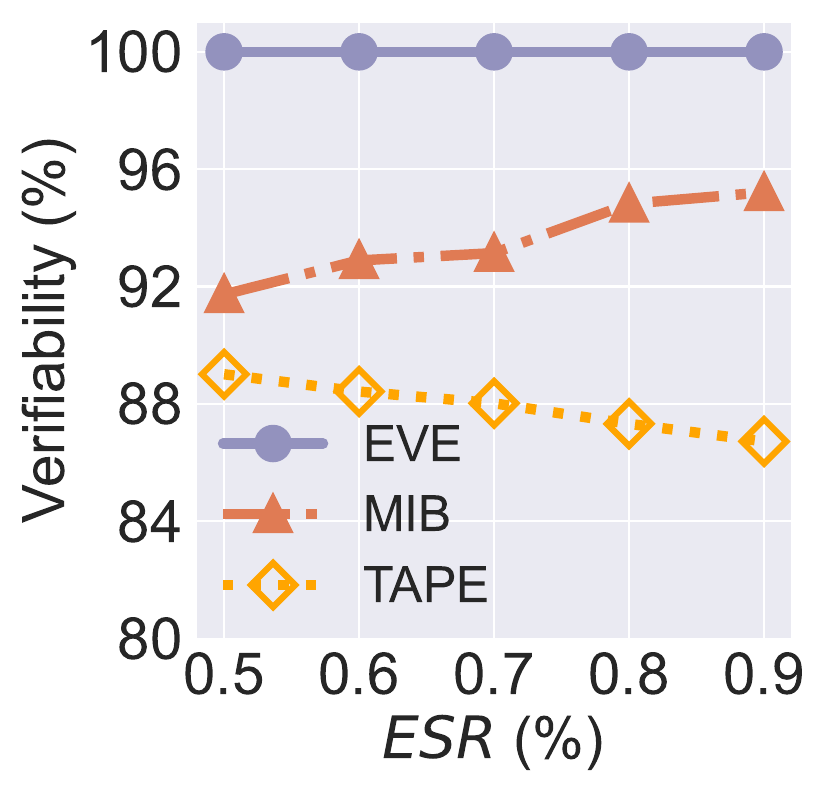}
		%	\caption{Another example of a subfigure.}
		\label{fig:celebaverifiabilitysamplesizesimanalysis}
	}
	\subfloat{ 	 
		\includegraphics[scale=0.28]{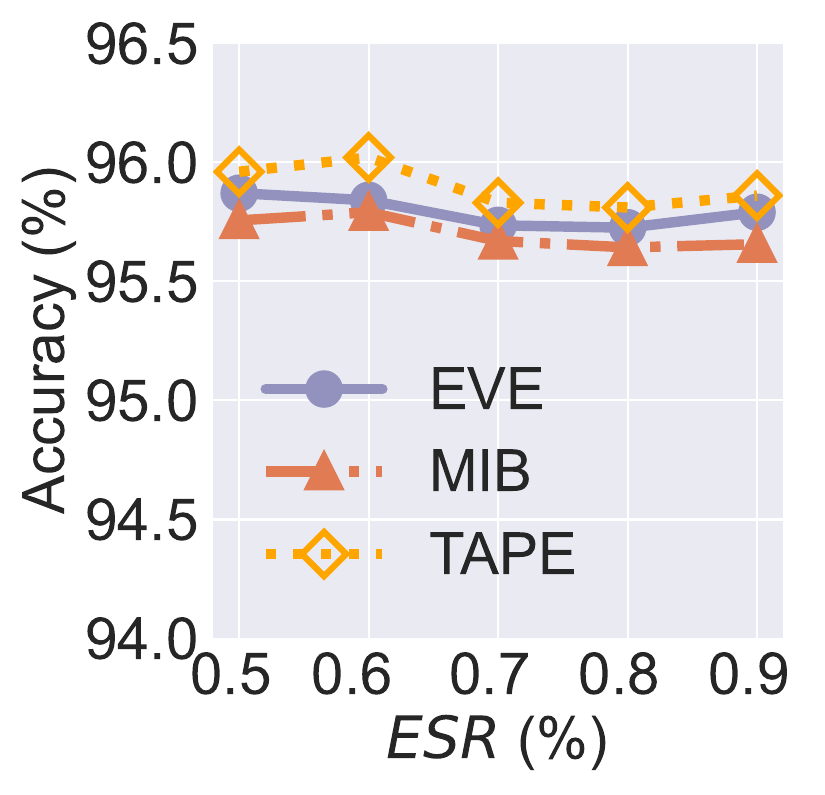}
		%	\caption{Another example of a subfigure.}
		\label{fig:celebarecsimsamplesizesimanalysis}
	} 
	\vspace{-2mm}
	\caption{Evaluations of impact about different $\text{\it ESR}$.} 
	\label{evaluation_of_in_or_not_in} 
\end{figure}

 \noindent
\textbf{Verification Effect and Functionality Preservation.} 
\Cref{evaluation_of_in_or_not_in} evaluates the impact of different {\it ESR} on the verifiability (left half) and the accuracy of the unlearned models (right half)  for EVE, MIB, and TAPE on MNIST, CIFAR10, and CelebA. EVE and MIB show near-perfect verifiability across all MNIST, CIFAR10 and CelebA, while the verifiability of TAPE drops significantly when {\it ESR} increases. It is because TAPE analyzes the unlearned information from model differences, while more erasure information included in one update model difference increases the difficulty for analysis. Since TAPE is also not involved in the unlearning process, it achieves the best model accuracy. EVE consistently maintains higher accuracy than MIB across all datasets and {\it ESR} values. Especially on MNIST, EVE keeps a near-constant accuracy (98\%) across all {\it ESR} levels, whereas MIB exhibits a slight decline at higher {\it ESR}s. Taken together, EVE uniquely delivers perfect verification of data removal without sacrificing predictive performance, while MIB offers moderate verification gains with stable accuracy, and TAPE prioritizes utility at the cost of weaker verifiability.

%Due the space limitations, we put the results on CelebA in \Cref{additional_eff}.

\iffalse
\begin{tcolorbox}[colback=white, boxrule=0.3mm]
	\noindent \textbf{Takeaway 2.} 
The running time of CAP slightly increases as {\it ESR} increases, but CAP still achieves significant efficiency improvement because it is independent of the original ML service model training. While larger {\it ESR} decreases the unlearned model utility, CAP provides similar or even better unlearning verification and accuracy than MIB. 
\end{tcolorbox}
\fi

 \begin{figure*}[t]
	\centering
	\subfloat{ 	%\rotatebox{90}{\hspace{-5mm} \scriptsize{ On MNIST} }
		\includegraphics[scale=0.3]{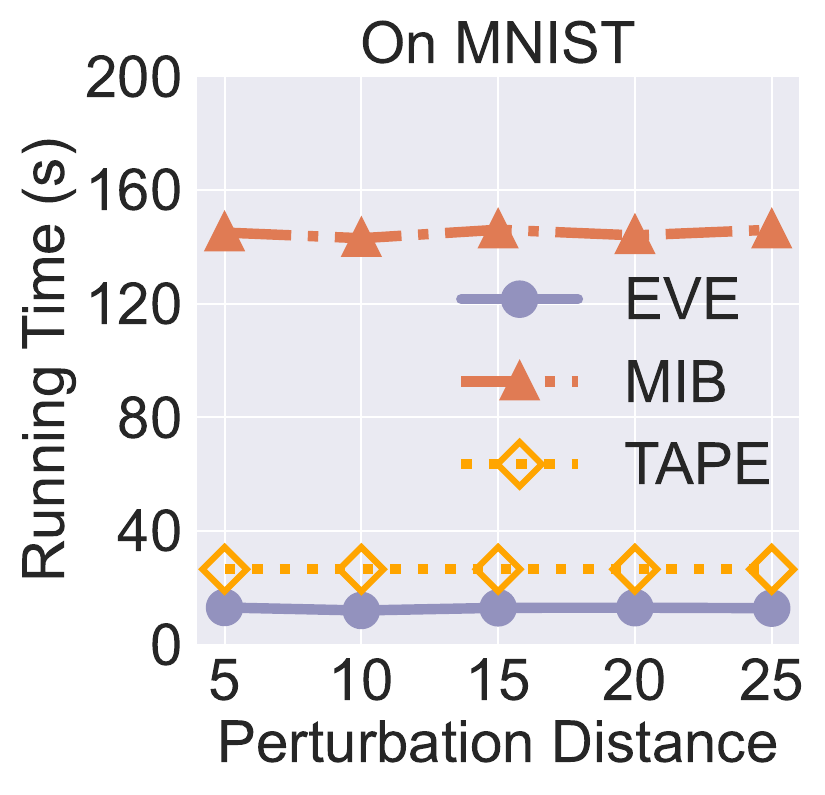}
		%\caption{An example of a subfigure.}
		\label{fig:mnistrunningtimenoiseanalysisbar}  
	} 
	\subfloat{ 	%\rotatebox{90}{ \hspace{-7mm}	\scriptsize{ On CIFAR10} }
		\includegraphics[scale=0.3]{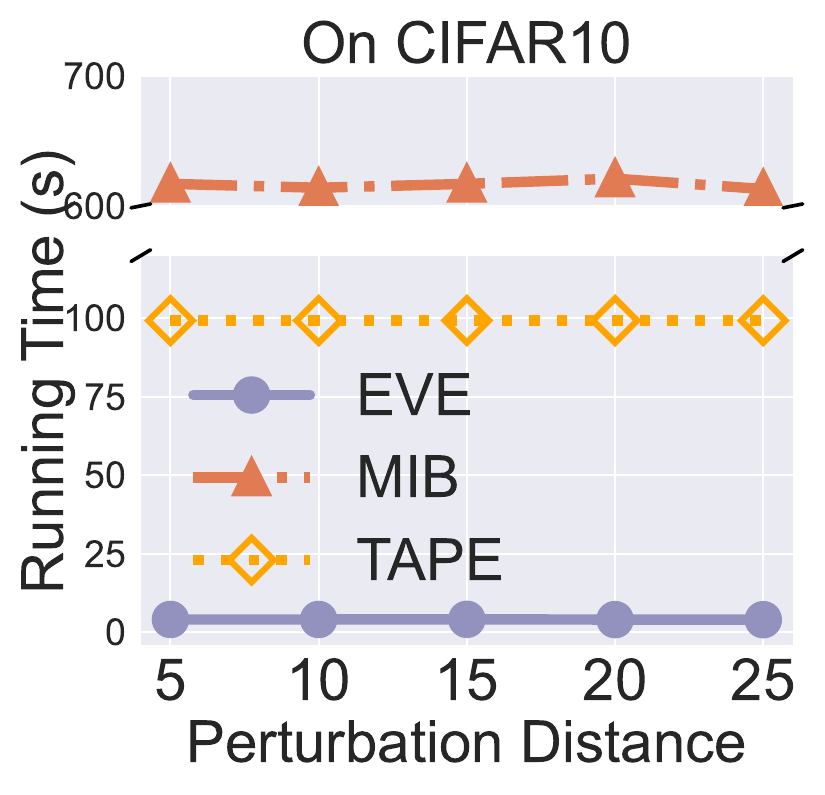}
		%\caption{An example of a subfigure.}
		\label{fig:cifar10runningtimenoiseanalysisbar} 
	}
	\subfloat{  %\rotatebox{90}{ \hspace{-3mm}	\scriptsize{ On STL-10} }
		\includegraphics[scale=0.3]{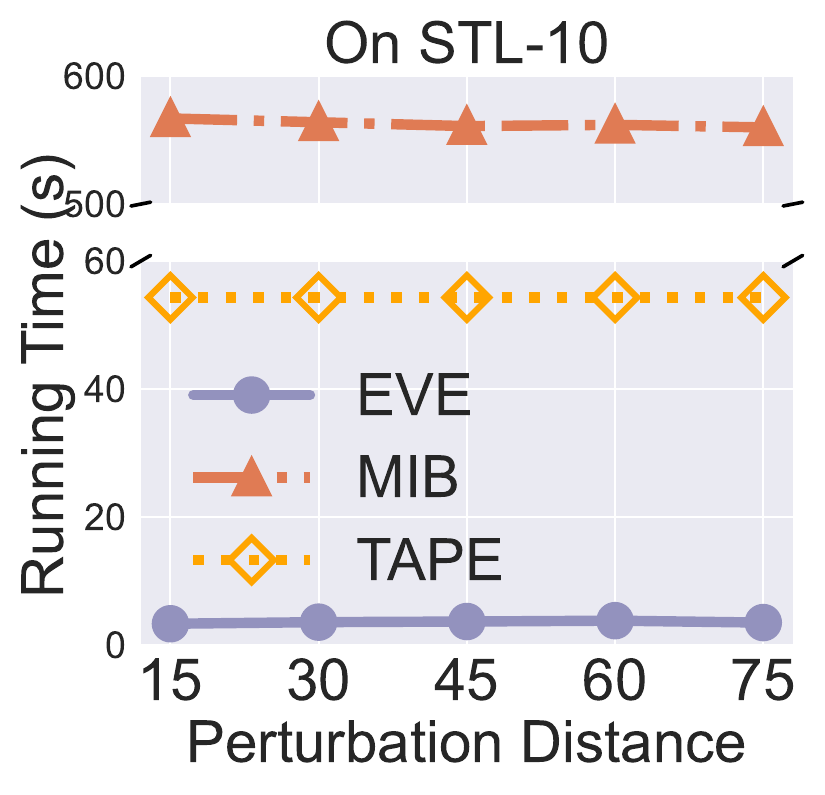}
		%\caption{An example of a subfigure.}
		\label{fig:stl10rtnoiseanalysisbar} 
	}
	\subfloat{ %\rotatebox{90}{ \hspace{-3mm}	\scriptsize{ On CelebA} }
		\includegraphics[scale=0.3]{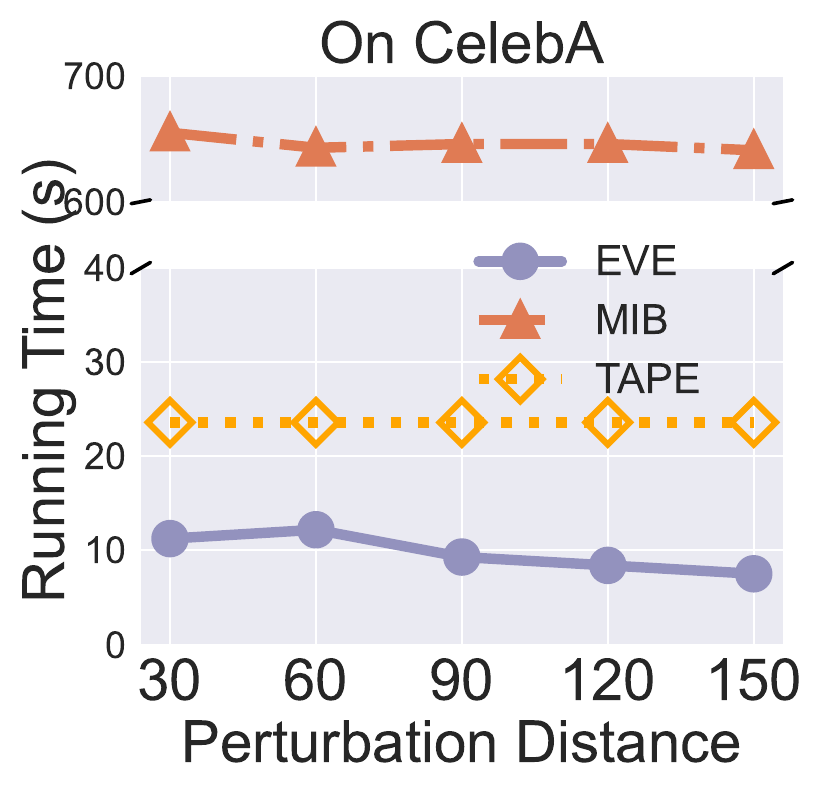}
		%\caption{An example of a subfigure.}
		\label{fig:celebarunningtimenoiseanalysisbar}
	} \vspace{-4mm}
	\\
	\subfloat{ 	
		\includegraphics[scale=0.3]{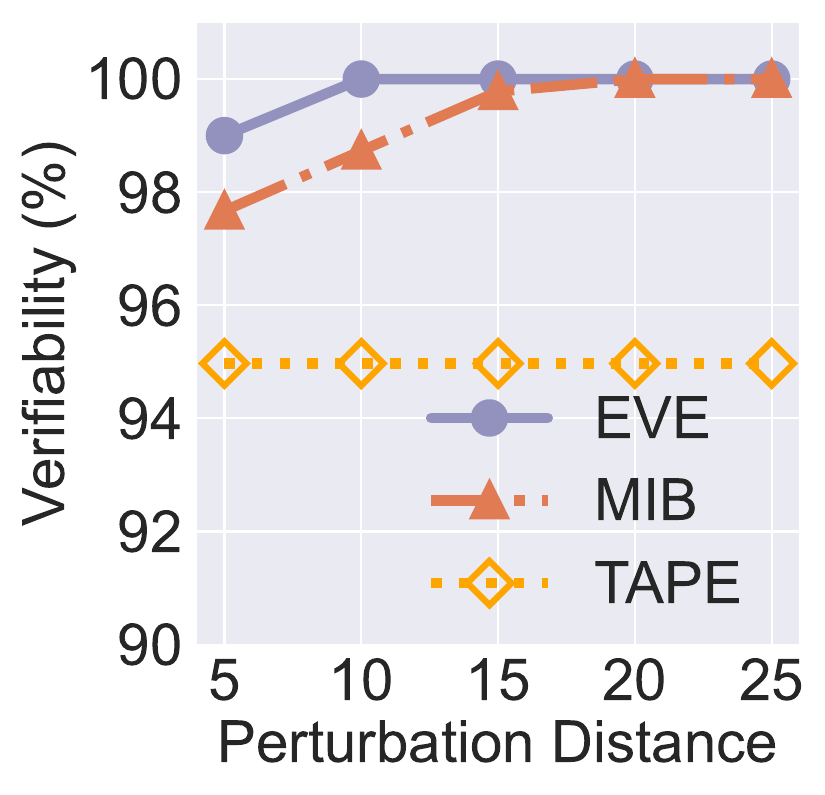}
		%	\caption{Another example of a subfigure.}
		\label{fig:mnistverifiabilitynoiseanalysis}
	} 
	\subfloat{ 
		\includegraphics[scale=0.3]{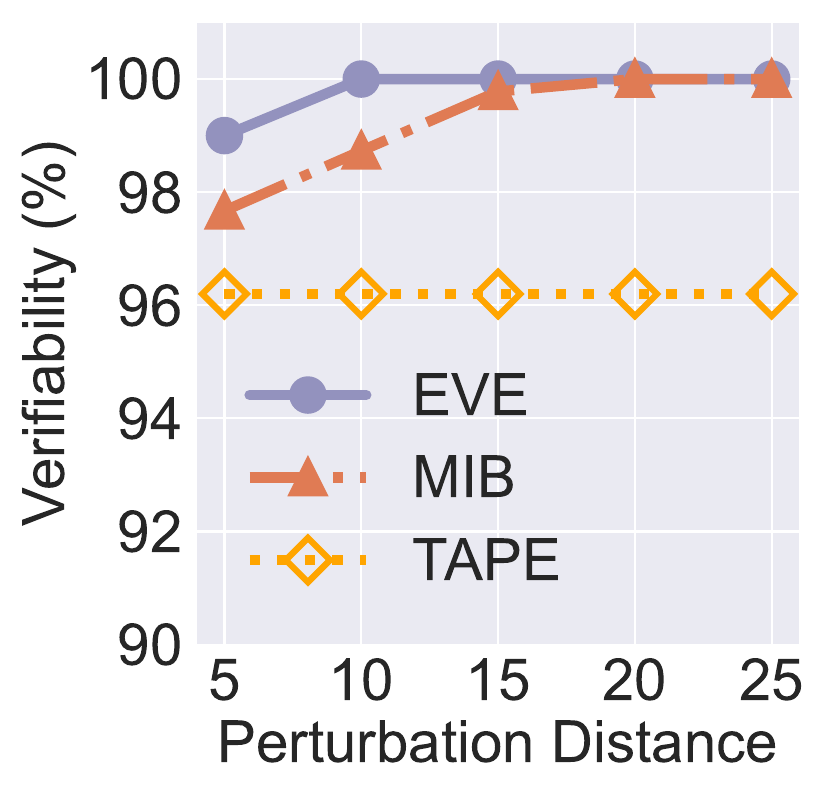}
		%	\caption{Another example of a subfigure.}
		\label{fig:cifar10verifiabilitynoiseanalysis}
	}
	\subfloat{ 
		\includegraphics[scale=0.3]{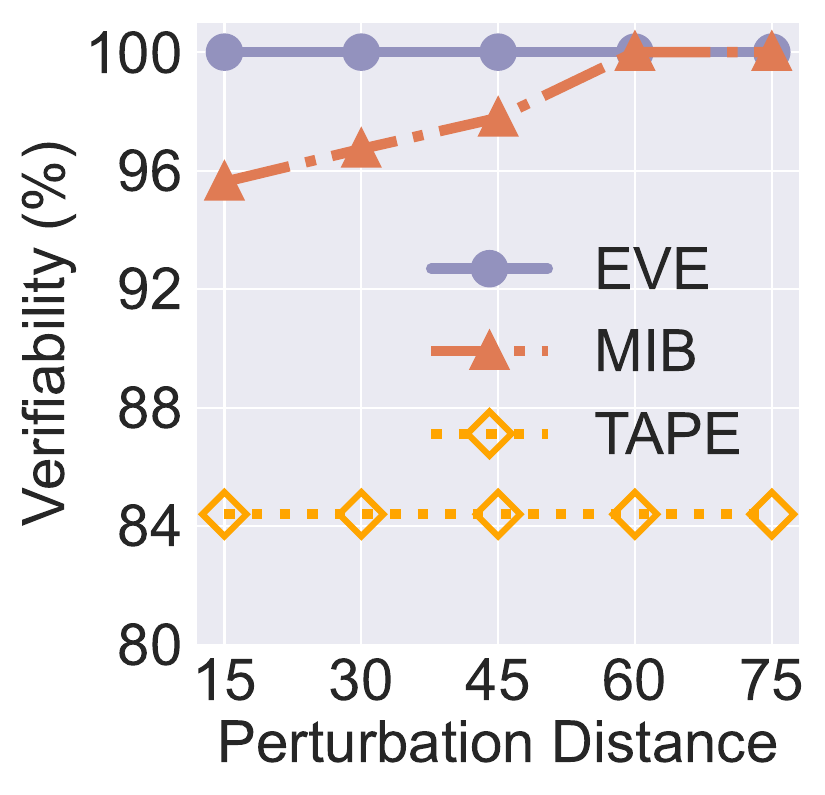}
		%	\caption{Another example of a subfigure.}
		\label{fig:stl10verifiabilitynoiseanalysis}
	}
	\subfloat{ 
		\includegraphics[scale=0.3]{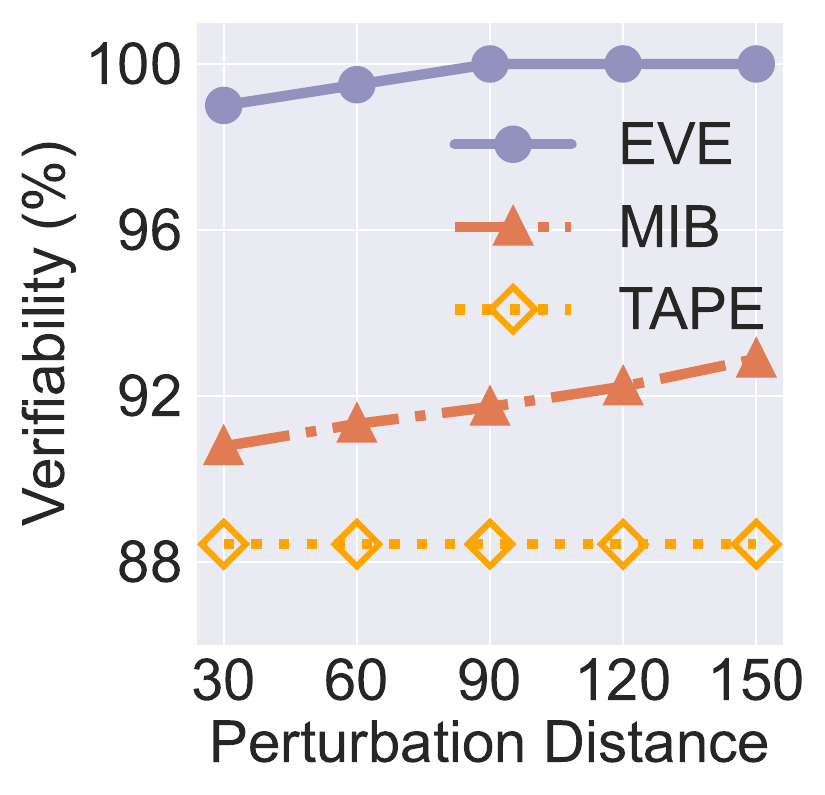}
		%	\caption{Another example of a subfigure.}
		\label{fig:celebaverifiabilitynoiseanalysis}
	} \vspace{-4mm}
	\\
	\subfloat{ 	
		\includegraphics[scale=0.3]{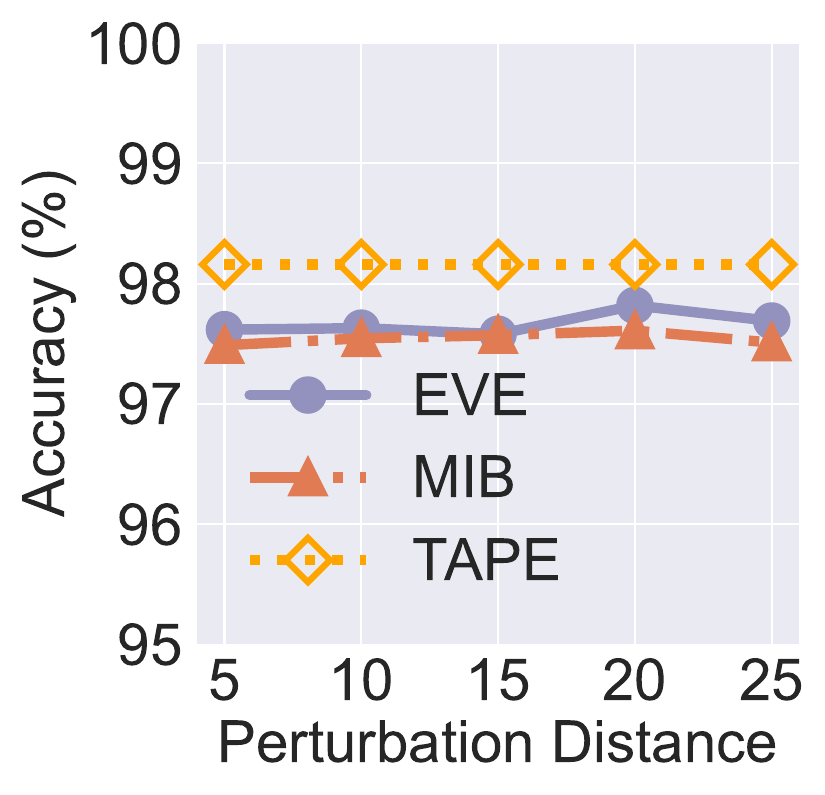}
		%\caption{An example of a subfigure.}
		\label{fig:mnistrecmsenoiseanalysis}
	}
	\subfloat{ 
		\includegraphics[scale=0.3]{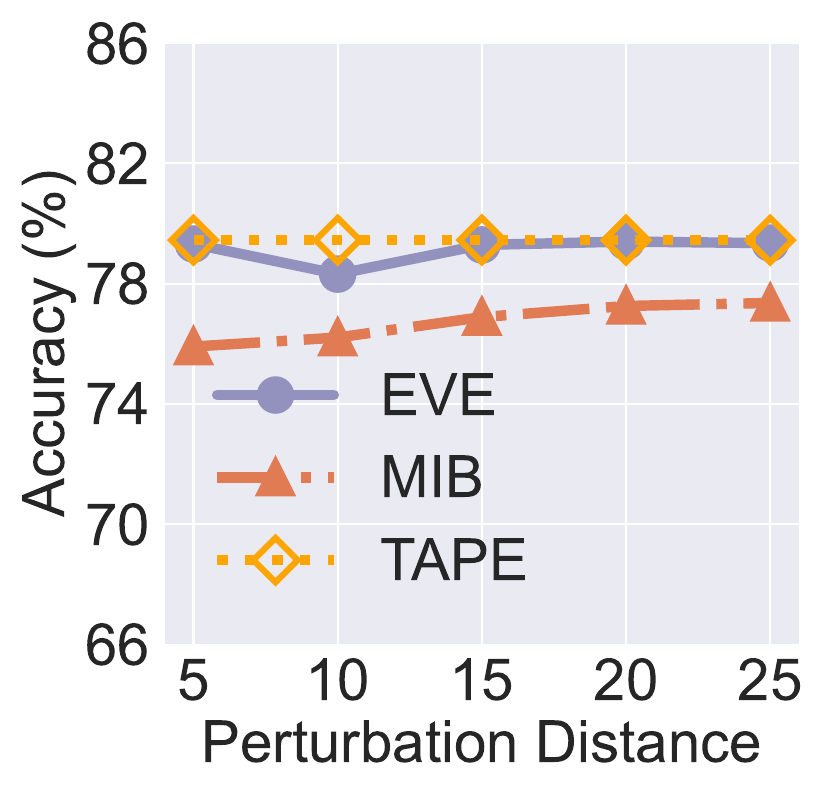}
		%\caption{An example of a subfigure.}
		\label{fig:cifar10recmsenoiseanalysis}
	}
	\subfloat{ 
		\includegraphics[scale=0.3]{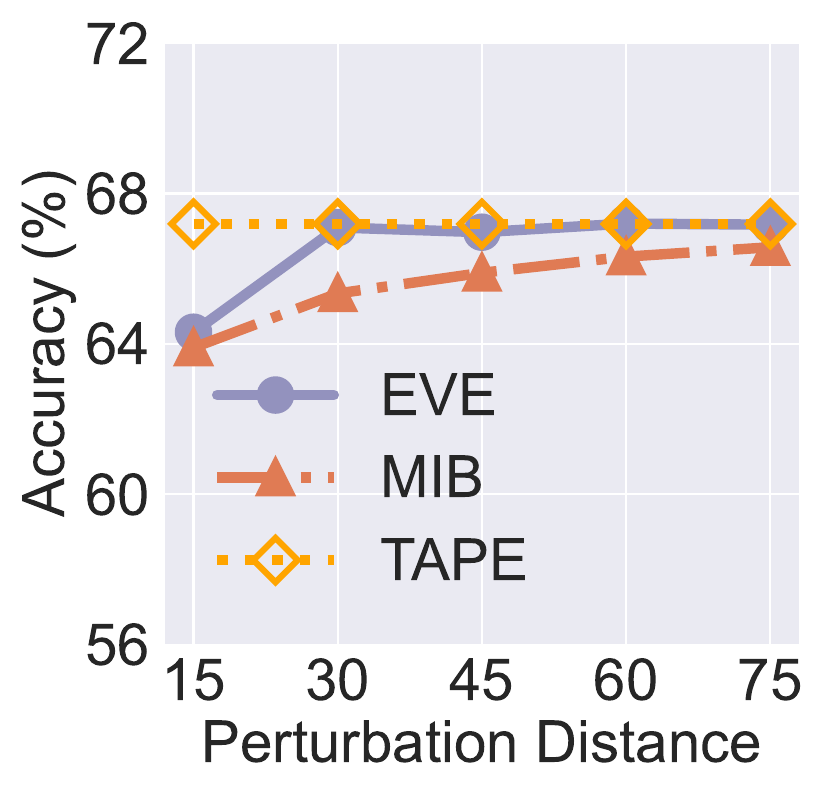}
		%\caption{An example of a subfigure.}
		\label{fig:stl10recmsenoiseanalysis}
	}
	\subfloat{ 
		\includegraphics[scale=0.3]{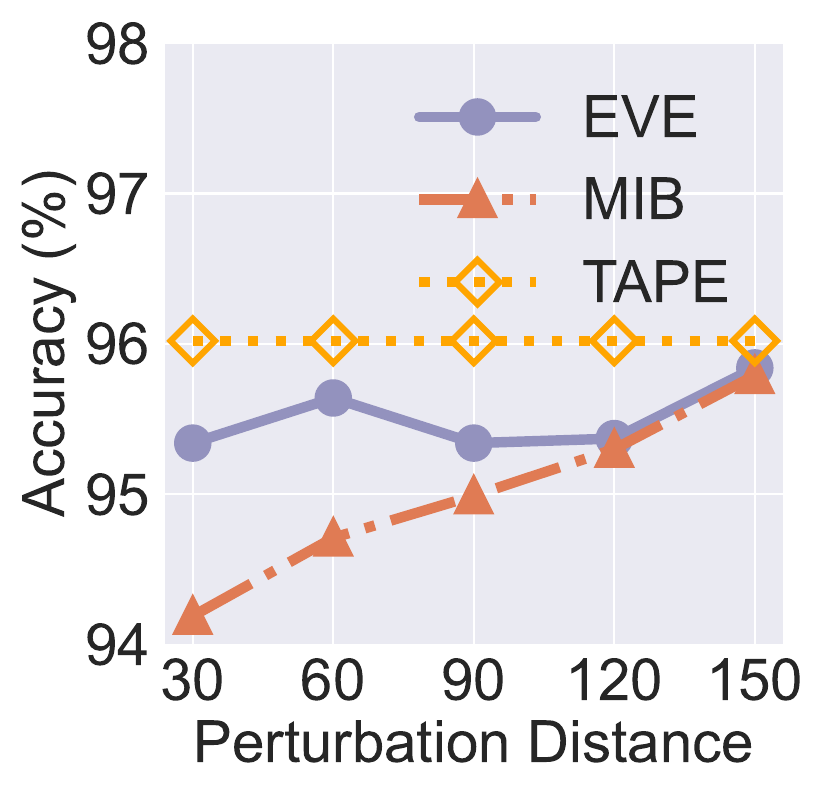}
		%\caption{An example of a subfigure.}
		\label{fig:celebarecmsenoiseanalysis}
	}
	\vspace{-2mm}
	\caption{Evaluations of the impact of the unlearned data perturbation distance $d$. } 
	\label{evaluation_of_noise} 
%	\vspace{-4mm}
\end{figure*}

\subsection{Ablation Study of Perturbation Distance} \label{adding_noise}

%In this section, we evaluate the impact of the perturbation of the unlearned data. We find that only perturbation without changing the labels as backdoor-based methods already significantly improves the information reconstruction and assists the verifiability of data removal.

\noindent
\textbf{Setup.} 
The customized perturbation algorithm outputs $D_{u,p} \gets (X_u + \delta_u, Y_u)$, where the perturbation is constrained by $\| \delta \|_{\infty} \leq d$. We set the perturbation distance $d$ between 0 and 25 for MNIST and CIFAR10, 0 and 75 for STL-10, and 0 and 150 for CelebA, with the range determined by the size of the dataset. We keep all other parameters fixed while varying the perturbation distance $d$. For the MIB method, we control the backdooring trigger patches with the same changing distance as our methods to ensure they can be compared. For the TAPE method, since it has no perturbation, we directly show the results of the fixed parameters. %\Cref{evaluation_of_noise} illustrates the effect of varying the unlearned data perturbation distance $d$ on running time, accuracy, and verifiability across four datasets.

\noindent
\textbf{Impact on Efficiency.} 
The running time plots (first row of~\Cref{evaluation_of_noise}) show that EVE consistently outperforms MIB and TAPE in terms of efficiency across all datasets, and EVE's performance remains relatively stable even as the perturbation distance increases. MIB consumes the most computation cost as it needs to involve the initial training process.

%\noindent
%\textbf{Relationship between Data Similarity and Perturbation Limit Value.} The first column in \Cref{evaluation_of_noise} shows the relationship between the perturbation limit $\alpha$ and data similarity, where the similarity is between the original image data $D_{u}$ and the perturbed data $D_{u, p}$. On all the datasets, adding more perturbation will decrease the similarity. 

\noindent
\textbf{Impact on Verification Effect and Functionality Preservation.} 
The second row of \Cref{evaluation_of_noise} illustrates the verifiability of unlearning as the perturbation distance increases. Higher perturbation distance definitely results in a better verification effect. Both EVE and MIB have better verifiability when the perturbation distance increases. EVE provides more reliable unlearning verification, consistently outperforming MIB and TAPE even under varying perturbation distances.

The third row of \Cref{evaluation_of_noise} highlights the influence of verification methods on unlearned models as the perturbation distance increases. Dislike perturbation for machine learning, more perturbation for erasure samples for unlearning preserves the unlearned model utility better. Both EVE and MIB have an accuracy increase when the perturbation increases. EVE demonstrates superior accuracy preservation than MIB across most datasets. Since TAPE doesn't involve any model training, it does not influence the unlearned model and achieves the best model accuracy.

\iffalse

\begin{tcolorbox}[colback=white, boxrule=0.3mm]
	\noindent \textbf{Takeaway 3.} 
	A larger unlearned data perturbation limitation $\alpha$ does not effectively influence the efficiency and verifiability but slightly increases the functionality preservation.
\end{tcolorbox}

\fi

	\section{Summary and Future Work} \label{s_a_fw}

In this paper, we propose the EVE method for verifying machine unlearning, focusing exclusively on the unlearning process. EVE works by perturbing the unlearning data to embed a verification signal into the unlearned model. This signal is that the unlearned model mispredicts target samples that the trained model had previously classified correctly. Users can utilize the changes of prediction for customized target samples to infer if their unlearning requests are conducted.

The EVE method proposed in this paper overcomes key limitations of existing unlearning verification approaches by eliminating the need for involvement in the original model training process. However, the limitation is that EVE is only suitable for approximate unlearning verification. Future research can build on this work by developing even more efficient unlearning verification methods suitable for broader scenarios.

%We formalize the perturbation generation as an adversarial optimization problem and introduce a perturbation descent method based on gradient matching to solve it. 

%Extensive experimental results demonstrate that CAP significantly improves efficiency compared to backdoor-based methods while also providing effective unlearning verification.

%, further enhancing the ability to guarantee and support the right to be forgotten in MLaaS environments.

	\begin{acks}
		The authors would like to thank the anonymous reviewers for their valuable comments and review. 
	\end{acks}
	
	\bibliographystyle{ACM-Reference-Format}
	\bibliography{main_EVA}

	%\newpage
	\clearpage
	\appendix

\section{Proof of Theorem 1} \label{proof_of_theorem1}

%We give the proof of Theorem 1. We analyze under what conditions the unlearning users can reject the no-unlearning hypothesis $\mathcal{H}_{0}$ to claim that their was unlearned by the server.

We provide the proof for Theorem 1, where we analyze the conditions under which unlearning users can reject the null hypothesis, $\mathcal{H}_{0}$ (no-unlearning), to claim that their data has been successfully unlearned by the server.

\setcounter{theorem}{0}
\begin{theorem} 
	Given a target model $f(\cdot)$ in a classification task with $K$ classes, and $m$ queries to $f(\cdot)$, if the model's misprediction probability $\alpha$ for the target sample $f(x_t) != y_{t}$ satisfies the following formula:
	\begin{equation}
		\sqrt{m-1}  \cdot (\alpha - \beta) - \sqrt{\alpha -\alpha^2} \cdot t_{\tau}> 0,
	\end{equation}
	the unlearning user can reject the null hypothesis $\mathcal{H}_{0}$ at significance level $1-  \tau$, 
	where $\beta = \frac{K-1}{K}$ is the expected misprediction probability under random chance, and $t_{\tau}$ is the $\tau$ quantile of the $t$ distribution with $m-1$ degrees of freedom.
\end{theorem}

\begin{proof}
	Consider a target model $f(\cdot)$ in a classification task with $K$ classes. Let $x_t$ be the verification query target, we denote the misprediction results as $\mathcal{R}$, which is a random variable and follows the binomial distribution
	\begin{equation}
		\mathcal{R} \sim B(1, q),
	\end{equation}	
	where $q = \Pr (f(x_t) != y_t)$ representing the misprediction probability for the target sample. The data owner uses multiple target samples $x_1, x_2, ..., x_m$ to query the target model and receives their misprediction results $\mathcal{R}_1, \mathcal{R}_2, ..., \mathcal{R}_m$. The total misprediction rate is calculated as follows:
	\begin{equation}
		\alpha = \frac{\mathcal{R}_1 + \mathcal{R}_2 + \cdot\cdot\cdot + \mathcal{R}_m}{m}.
	\end{equation}
	Because $\mathcal{R}_1, \mathcal{R}_2, ..., \mathcal{R}_m$ are iid random variables, $\alpha$ follows the distribution:
	\begin{equation}
		\alpha \sim \frac{1}{m} B(m, q).
	\end{equation}
	According to the Central Limit Theorem (CLT) \cite{montgomery2010applied}, when $m \geq 30$, $\alpha$ follows the normal distribution:
	\begin{equation}
		\alpha \sim \mathcal{N} (q, \frac{q \cdot (1-q)}{m}).
	\end{equation}
	Because $\alpha$ follows a normal distribution, we can use a T-test \cite{montgomery2010applied} to determine if $q$ is significantly different from the original model prediction. We construct the t-statistic as follows:
	\begin{equation}
		T=\frac{ \sqrt{m}(\alpha - \beta )}{s},
	\end{equation}
	where $\beta = \frac{K-1}{K}$, and $s$ is the standard deviation of $\alpha$. $s$ is calculated as follows:
	\begin{equation}
		\begin{aligned}
			s^2 &= \frac{1}{m-1} \sum_{i=1}^{m} (\mathcal{R}_i - \alpha)^2 \\
			& = \frac{1}{m-1} (\sum_{i=1}^m \mathcal{R}_i^2 -\sum_{i=1}^m 2\mathcal{R}_i \cdot \alpha + \sum_{i=1}^m \alpha^2 ) \\
			& = \frac{1}{m-1} (\sum_{i=1}^m \mathcal{R}_i^2  - 2m \cdot \alpha^2 + m \cdot \alpha^2) \\
			& = \frac{1}{m-1} (\sum_{i=1}^m \mathcal{R}_i^2  - m \cdot \alpha^2).
		\end{aligned}
	\end{equation}
	Since the misprediction query is Binary (0 or 1), we have $\mathcal{R}_i^2= \mathcal{R}_i$. Then, we have 
	\begin{equation} \label{s_eq}
		\begin{aligned}
			s^2 & = \frac{1}{m-1} (m \cdot \alpha - 2m \cdot \alpha^2 + m \cdot \alpha^2) \\
			& = \frac{1}{m-1} (m \cdot \alpha - m \cdot \alpha^2).
		\end{aligned}
	\end{equation}
	Under the null hypothesis, the T statistic follows a t-distribution with $m-1$ degree of freedom. At the significant level $1 - \tau$, if the following inequality formula holds, we can reject the null hypothesis $\mathcal{H}_{0}$:
	\begin{equation} \label{leq_eq}
		\frac{\sqrt{m} (\alpha - \beta)}{ s} > t_{\tau},
	\end{equation}
	where $ t_{\tau}$ is the ${\tau}$ quantile of the $t$ distribution with $m-1$ degrees of freedom. Combine \Cref{s_eq} and \Cref{leq_eq}, we can get 
	\begin{equation}
		\sqrt{m-1}  \cdot (\alpha - \beta) - \sqrt{\alpha -\alpha^2} \cdot t_{\tau}> 0,
	\end{equation}
	which concludes the proof.

\end{proof}

\section{Detailed Introduction of Datasets} \label{datasets_appendix}

\begin{table}[h]
	\scriptsize
	\caption{Dataset statistics.}
	\label{dataset_table}
%	\vspace{-3mm}
	\resizebox{\linewidth}{!}{
		\setlength\tabcolsep{7.5pt}
		\begin{tabular}{cccc}
			\toprule[0.8pt]
			Dataset & Feature Dimension  & \#. Classes & \#. Samples \\
			\midrule
			MNIST & 28×28×1 & 10 & 70,000  \\  
			\rowcolor{verylightgray}
			CIFAR10 & 32×32×3 & 10 & 60,000  \\  
			STL-10 & 96x96x3 &10 & 5000 \\
			\rowcolor{verylightgray}
			CelebA & 178×218×3 & 2 (Gender) & 202,599 \\
			\bottomrule[0.8pt]
	\end{tabular}}
\end{table}

The statistics of all datasets used in our experiments are listed and introduced in \Cref{dataset_table}. Our experiment on CelebA is to identify the gender attributes of the face images. The task is a binary classification problem, different from the ones on MNIST, CIFAR10 and STL-10. We also introduce them below

\begin{itemize}
	\item \textbf{MNIST.} MNIST contains 60,000 handwritten digit images for the training and 10,000 handwritten digit images for the testing. All these black and white digits are size normalized, and centered in a fixed-size image with 28 × 28 pixels.
	\item \textbf{CIFAR10.} CIFAR10 dataset consists of 60,000 32x32 colour images in 10 classes, with 6,000 images per class. There are 50,000 training images and 10,000 test images.
	\item \textbf{STL-10.} STL-10 dataset consists of 13,000 color images with 5,000 training images and 8,000 test images. STL-10 has 10 classes of airplanes, birds, cars, cats, dear, dogs, horses, monkeys, ships, and trucks with each image having a higher resolution of 96x96 pixels. Compared to the above two datasets, STL-10 can be considered as a more challenging dataset with higher learning complexity.
	\item  \textbf{CelebA.} CelebA is a large-scale face attributes dataset with more than 200,000 celebrity images, each with 40 attribute annotations, and the size of each image is 178×218.
\end{itemize}

%\section{Benchmarks} \label{benchmarks_intro}

% We briefly summarize these unlearning benchmarks as follows.

\begin{table}[h]
	\scriptsize
	\caption{Evaluating mixing other users' unlearned data in the unlearning operation on CIFAR10. }
	\label{Evaluating_mixing}
	\resizebox{\linewidth}{!}{
		\setlength\tabcolsep{7.5pt}
		\begin{tabular}{c|c|c|c|c|c|c}
			\toprule[0.8pt]
			\midrule
			On CIFAR10 & Add 0\% & 20\% & 40\% & 60\% & 80\% & 100\% \\
			\midrule
			Running time (s)      & 3.93     & 4.32  &  4.74  &  5.19  & 5.63  &  6.10  \\  
			\rowcolor{verylightgray}
			Unl. Verifiability (\%)    & 100.00\%       & 100.00\%   & 100.00\%   &  100.00\%  &  100.00\%  &  100.00\%  \\  
			Accuracy of unelarned model (\%)     &  79.35\%        &  78.34\%     &  74.87\%     &  75.38\%  &  76.33\%  & 75.78\% \\
			\midrule
			\bottomrule[0.8pt]
	\end{tabular}}
\end{table}

\begin{figure*}[t]
	\centering
	\subfloat{ 	
		\includegraphics[scale=0.4]{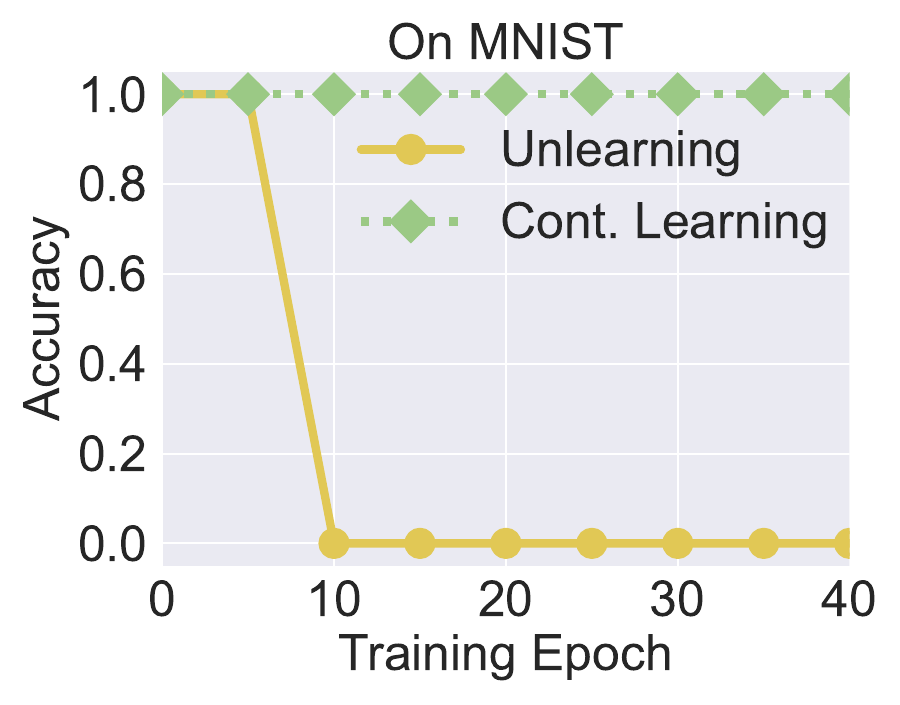}
		%\caption{An example of a subfigure.}
		\label{fig:mnisttrainingaccshift}
	}
	\subfloat{ 	
		\includegraphics[scale=0.4]{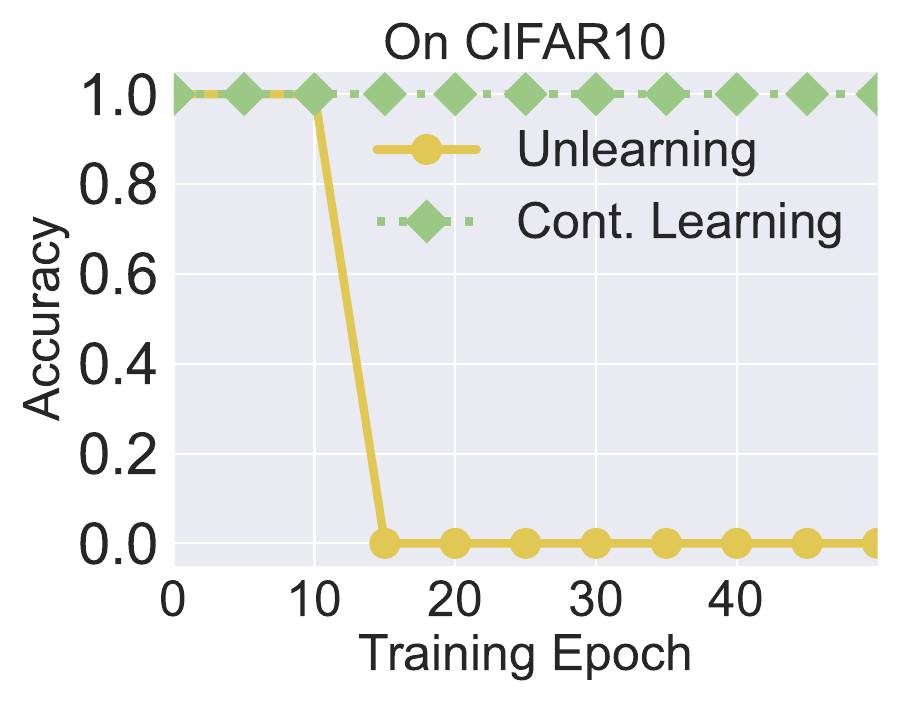}
		%\caption{An example of a subfigure.}
		\label{fig:cifar10trainingaccshift} 
	}
	%	\vspace{-6mm}
	\caption{Model decision boundary shift for the target sample.} 
	\label{acc_shift_for_unl} 
	%	\vspace{-4mm}
\end{figure*}

\section{Additional Experimental Results}

\subsection{Additional Experimental Results about Scalability}   \label{add_experiment_of_scala}

To demonstrate the scalability of our method when other users upload unlearning requests at the same time, we additionally conduct experiments that mix other users' unlearning data in one unlearning request for the unlearning operation. In this experiment, we set the ESR=2\% for our unlearning verification user, and we mix additional samples from 0\% to 100\% compared to the uploaded samples of the unlearning verification user. We present the experimental results on CIFAR10 in \Cref{Evaluating_mixing}.

The results demonstrate that mixing the unlearned samples of other users in the unlearning operation will degrade the model utility as more samples are unlearned. These results also show that the verification method still works effectively when mixing the same size of additional unlearning samples as the verification user's uploaded dataset.

\subsection{Exclusive for Unlearning Method }

We also evaluate whether the perturbed data is exclusively effective for unlearning methods in \Cref{acc_shift_for_unl}. We test the perturbed samples using unlearning and continual learning methods and record the accuracy of the target samples. Results in \Cref{acc_shift_for_unl} on both MNIST and CIFAR10 show that the decision boundary shifts only when unlearning methods are applied to these perturbed samples. In contrast, the continual learning methods on these samples will not change the model decision boundary. The reason is that the perturbation has not dramatically changed the information of the original samples, and we have not changed the labels like backdoor methods.

%\subsection{Additional Evaluation on Various \textit{ESR} } \label{additional_eff}

\end{document}